\documentclass[twoside]{article}

\usepackage[accepted]{aistats2021}
\usepackage[T1]{fontenc}    
\usepackage{booktabs}       
\usepackage{nicefrac}       
\usepackage{microtype}      
\usepackage{wrapfig}

\usepackage{booktabs} 

\usepackage{amsfonts}
\usepackage{amsmath}
\usepackage{amssymb}
\usepackage{graphicx}
\usepackage{amsthm}
\usepackage[dvipsnames]{xcolor}

\usepackage{multirow}
\usepackage{comment}
\usepackage{verbatim}
\usepackage{caption}
\usepackage{subcaption}

\usepackage[round]{natbib}

\definecolor{SynBlue}{HTML}{004488}
\definecolor{SynYellow}{HTML}{ddaa33}
\definecolor{SynRed}{HTML}{bb5566}

\DeclareMathOperator{\diag}{diag}

\newcommand{\argmax}{\operatornamewithlimits{argmax}}

\newcommand{\given}{\,|\,}

\def\md{\,{\mathrm d}}

\def\bR{{\mathbb R}}

\def\NPDF{{\mathcal N}}

\def\f0{{\mathbf 0}}

\def\Tr{{\operatorname{Tr}}}
\def\IGPDF{{\mathcal{IG}}}
\def\TPDF{{\mathcal T}}
\def\KL{{\operatorname{KL}}}
\def\vect{{\operatorname{vec}}}

\usepackage{mathtools}

\newtheorem{defn}{Definition}

\newtheorem{lem}{Lemma}

\newtheorem{prop}{Proposition}

\newtheorem{propty}{Property}

\theoremstyle{definition}

\newtheorem{remark}{Remark}

\usepackage[dvipsnames]{xcolor}
\definecolor{bred}{rgb}{0.8,0,0}
\usepackage{hyperref}
\hypersetup{colorlinks,linkcolor={blue},citecolor={bred},urlcolor={blue}}
\usepackage[noend]{algpseudocode}
\usepackage{algorithmicx}
\usepackage{algorithm}

\begin{document}

%

%

\newcommand{\blindfootnote}[1]{{\renewcommand\thefootnote{$^{\star}$}\footnotetext{#1}}}

\runningauthor{Akyildiz, Van den Burg, Damoulas, Steel}

\twocolumn[%
\aistatstitle{Probabilistic Sequential Matrix Factorization}
\aistatsauthor{
\"Omer Deniz Akyildiz$^{\star,1}$ \And Gerrit J.J. van den Burg$^{\star,1}$ \AND Theodoros Damoulas$^{1,2}$ \And Mark F.J. Steel$^{2}$}

\aistatsaddress{$^{1}$The Alan Turing Institute, London, UK \And $^{2}$University of Warwick, Coventry, UK}
]

\blindfootnote{Joint first authorship.} 

\begin{abstract}
We introduce the probabilistic sequential matrix factorization (PSMF) method for factorizing time-varying and non-stationary datasets consisting of high-dimensional time-series. In particular, we consider nonlinear Gaussian state-space models where sequential approximate inference results in the factorization of a data matrix into a \textit{dictionary} and time-varying \textit{coefficients} with potentially nonlinear Markovian dependencies. The assumed Markovian structure on the coefficients enables us to encode temporal dependencies into a low-dimensional feature space. The proposed inference method is solely based on an approximate extended Kalman filtering scheme, which makes the resulting method particularly efficient. PSMF can account for temporal nonlinearities and, more importantly, can be used to calibrate and estimate generic differentiable nonlinear subspace models. We also introduce a robust version of PSMF, called rPSMF, which uses Student-t filters to handle model misspecification. We show that PSMF can be used in multiple contexts: modeling time series with a periodic subspace, robustifying changepoint detection methods, and imputing missing data in several high-dimensional time-series, such as measurements of pollutants across London.
\end{abstract}

\section{INTRODUCTION}
\label{sec:intro}
The problem of $r$-rank factorization of a data matrix $Y \in \bR^{d\times n}$ as
\begin{align}\label{eq:theMainProb}
Y \approx C X
\end{align}
with $C\in\bR^{d\times r}$ the \textit{dictionary matrix} and $X\in\bR^{r\times n}$ the  \textit{coefficients}, has received significant attention in past decades in multimedia signal processing and machine learning under the umbrella term of matrix factorization (MF) \citep{LeeSeungNMF,lee2001algorithms,mairal2010online}. The classical method for solving problems of the form \eqref{eq:theMainProb} is nonnegative matrix factorization (NMF) \citep{LeeSeungNMF}, which is proposed for nonnegative data matrices and obtains nonnegative factors. NMF and similar methods (e.g., singular value decomposition (SVD)) have been the focus of intensive research  \citep{lin2007projected,berry2007algorithms,ding2008convex,cai2010graph,fevotte2011algorithms} and has found applications in several fields, such as document clustering \citep{shahnaz2006document}, audio analysis \citep{smaragdis2003non,ozerov2009multichannel}, and video analysis \citep{BucakGunselNMF}. This work was extended to general MF problems for real-valued data and factors, which found many applications including collaborative filtering \citep{rennie2005fast} and drug-target prediction \citep{zheng2013collaborative}. 

The problem in \eqref{eq:theMainProb} was originally tackled from an optimization perspective, i.e., minimizing a cost $d(Y,CX)$ over $C$ and $X$ \citep{LeeSeungNMF,lee2001algorithms,lin2007projected,mairal2010online}. Another promising approach has been through a probabilistic model by defining priors on $C$ and $X$. This was explored by, e.g., \citet{cemgil09-nmf} with a Poisson-based model solved using variational inference (which reproduces NMF when the cost function is the Kullback-Leibler divergence) and in, e.g., \citet{mnih2008probabilistic} and \citet{salakhutdinov2008bayesian} with a Gaussian model for the real-valued case solved via Markov chain Monte Carlo (MCMC). 
Naturally, online versions of these methods have received significant attention as they enable scaling up to larger datasets. On this front, a number of algorithms have been proposed that are based either on stochastic optimization  \citep[e.g.,][]{BucakGunselNMF,mairal2010online,sismanisSGD,mensch2016dictionary} or on a probabilistic model for online inference \citep{wang2012probabilistic,paisley2014bayesian,akyildiz2019dictionary}. However, these methods are generally for i.i.d. data and cannot exploit the case where the columns of $Y$ possess time dependency.

The success of MF in the i.i.d. data case motivated the development of matrix factorization methods for time-dependent data. In this case, the problem can be formulated as inferring  parameters of a dynamical system or a state-space model (SSM), with a linear observation model $C$. This problem is also known as \textit{system identification} \citep{katayama2006subspace}. When the columns of $X$ (i.e. the hidden signal) also evolve linearly the problem reduces to that of inferring parameters of a linear SSM. This can be solved by maximum-likelihood estimation (MLE) through, e.g., expectation-maximization (EM) either offline \citep{ghahramani1996parameter,elliott1999new} or online \citep{cappe2009line,cappe2011online}, or with gradient-based methods \citep{andrieu2005line,kantas2015particle}. In this vein, \citet{yildirim2012online} address the NMF problem by introducing a SSM with a Poisson likelihood where the inference is carried out with sequential Monte Carlo (SMC). In a similar manner \citet{dynamicmatrixfact} propose an SSM-based approach where the dictionary is estimated using the EM algorithm. Similar MLE-based nonnegative schemes that also use SSMs attracted significant attention \citep[e.g.][]{mohammadiha2013prediction,mohammadiha2014state}.

Although MLE estimates are consistent in the infinite data limit for general SSMs \citep{douc2011consistency}, EM-based methods are prone to get stuck in local minima \citep{katayama2006subspace} and provide only point estimates rather than full posterior distributions. As an alternative to the EM-based approaches, optimization-based methods were also explored \citep[e.g.][]{boots2008constraint,karami2017multi,white2015optimal} which again result in point estimates. If the transition model for the coefficients $X$ exhibits nonlinear dynamics while the observation model is linear (by the nature of MF), the problem reduces to parameter estimation in nonlinear SSMs \citep{sarkka2013bayesian}. The MLE approach is again prominent in this setting using EM or gradient methods \citep{kantas2015particle}. However, when inference cannot be done analytically this results in the use of SMC \citep{doucet2000sequential} or particle MCMC methods \citep{andrieu2010particle} (see \citet{kantas2015particle} for an overview). Unfortunately, these methods suffer in the high-dimensional case \citep{bengtsson2008curse,snyder2008obstacles} which makes Monte Carlo-based methods unsuitable for solving the MF problem. Optimization-based approaches that formulate a cost function with temporal regularizers have also been studied \citep[e.g.][]{yu2016temporal,shi2016temporal,liu2016learning,ayed2019learning}.

An alternative to the MLE or optimization-based approaches is to follow a \textit{Bayesian} approach where a prior distribution is constructed over the parameters of the SSM, see, e.g., \citet{sarkka2013bayesian}. The goal is then to obtain the posterior distributions of the columns of $X$ and of $C$. This is also of interest when priors are used as regularizers to enforce useful properties such as sparsity \citep{cemgil09-nmf,schmidt2009bayesian}. In this context, an extension of the NMF-like decompositions to the dynamic setting was considered by \citet{fevotte2013non}, where the authors followed a maximum-a posteriori (MAP) approach. We refer to \citet{fevotte2018temporal} for a literature review of temporal NMF methods. However, these methods are \textit{batch} (offline) schemes and do not return a probability distribution over the dictionary or the coefficients. Joint posterior inference of $C$ and $X$ in a fully Bayesian setting is difficult as it usually requires sampling schemes \citep{salakhutdinov2008bayesian}. To the best of our knowledge, a fully Bayesian approach for sequential (online) inference for matrix factorization that also scales well with the problem dimension has not been proposed in the literature.
 


\paragraph{Contribution.} In this work, we propose the probabilistic sequential matrix factorization (PSMF) method by framing our matrix factorization model as a nonlinear Gaussian SSM. Our formulation is fully probabilistic in the sense that we place a matrix-variate Gaussian prior on the dictionary and use a general Markov model for the evolution of the coefficients. We then derive a novel approximate inference procedure that is based on extended Kalman filtering \citep{kalman1960new,mclean1962optimal} and results in a fast and efficient scheme. Our method is derived using numerical approximations to the optimal inference scheme and leverages highly efficient filtering techniques. 

In particular, we derive analytical approximations and do not require a sampling procedure to approximate the posterior distributions. The inference method we provide is explicit and the update rules can be easily implemented without further consideration on the practitioner's side. We also provide a robust extension of our model, called rPSMF, for the case where the model is misspecified and derive a corresponding inference scheme that adopts Student's t-filters \citep{giron1994bayesian,tronarp2019student}. Our methods can be easily tailored to the application at hand by modifying the subspace model, as the necessary derivatives can be easily computed through automatic differentiation.

This work is structured as follows. In Sec.~\ref{sec:model} we introduce our probabilistic state-space model and the robust extension. Next we develop our tractable inference and estimation method in Sec.~\ref{sec:InfAndEst}. In Sec.~\ref{sec:Experiments}, our method is empirically evaluated in different scenarios such as learning structured subspaces, multivariate changepoint detection, and missing data imputation. Sec.~\ref{sec:Conc} concludes the paper.



\paragraph{Notation} We denote the $d\times d$ identity matrix by $I_d$ and write $\NPDF(x;\mu,\Sigma)$ for the Gaussian density over $x$ with mean $\mu$ and covariance matrix $\Sigma$. Similarly, $\TPDF(x; \mu, \Sigma, \lambda)$ is the multivariate $t$ distribution with mean $\mu$, scale matrix $\Sigma$, and $\lambda$ degrees of freedom, and $\IGPDF(s; \alpha, \beta)$ is the inverse gamma distribution over $s$ with shape and scale parameters $\alpha$ and $\beta$. Further, $\mathcal{MN}(X; M, U, V)$ denotes the matrix-variate Gaussian with mean-matrix $M$, row-covariance $U$, and column-covariance $V$. Sequences are written as $x_{1:n} = \{x_1,\ldots,x_n\}$ and for a matrix $Z$, $z = \vect(Z)$ denotes vectorization of $Z$. Recall that if $C \sim \mathcal{MN}(C;M,U,V)$, then $c \sim \NPDF(c; \vect(M), V \otimes U)$ where $c = \vect(C)$ and $\otimes$ the Kronecker product \citep{matrixVariateDist}. With $y_k$ and $x_k$ we respectively denote the $k$-th column of the matrices $Y$ and $X$.


\section{THE PROBABILISTIC MODEL}\label{sec:model}

We first describe the SSM, which consists of observations $(y_k)_{k\geq 1} \in \bR^d$, latent coefficients $(x_k)_{k \geq 0} \in \bR^r$, and a latent dictionary matrix $C \in \bR^{d\times r}$, as follows
\begin{align}
p(C) &= \mathcal{MN}(C;C_0,I_d,V_0), \label{PriorC} \\
p(x_0) &= \NPDF(x_0; \mu_0, P_0), \label{PriorX} \\
p_\theta(x_k | x_{k-1}) &= \NPDF(x_k; f_\theta(x_{k-1}), Q_k), \label{TransitionX} \\
p(y_k | x_k, C) &= \NPDF(y_k; C x_k, R_k). \label{ObservationMod}
\end{align}
Here, $f_\theta \colon \bR^r \times \Theta \to \bR^r$ is a nonlinear mapping that defines the dynamics of the coefficients with $\Theta \subset \bR^{d_\theta}$ the parameter space, and  $(Q_k,R_k)_{k\geq 1}$ are respectively the noise covariances of the coefficient dynamics \eqref{TransitionX} and the observation model \eqref{ObservationMod}. The initial covariances of the coefficients and the dictionary are denoted by $P_0$ and $V_0$, respectively.

Intuitively, the model \eqref{PriorC}--\eqref{ObservationMod} is a dimensionality reduction model where the dynamical structure of the learned subspace is explicitly modeled via the transition density \eqref{TransitionX}. This means that inferring $C$ and $(x_k)_{k\geq 0}$ will lead to a probabilistic dimensionality reduction scheme where the dynamical structure in the data will manifest itself in the dynamics of the coefficients $(x_k)_{k\geq 0}$. One main difficulty for applying standard schemes in this case is that we assume $C$ to be an unknown and random matrix, therefore, the (extended) Kalman filter cannot be applied directly for inference. To alleviate this problem, we formulate the prior in \eqref{PriorC} with a Kronecker covariance structure, which enables us to update (conditional on $x_k$) the posterior distribution of $C$ analytically \citep{akyildiz2019dictionary}.

\subsection{The case of the misspecified model}
In the model \eqref{PriorC}--\eqref{ObservationMod}, when the practitioner does not have a good idea of how to set the hyperparameters or when they are misspecified, the resulting inference scheme may perform suboptimally. To remedy this situation and to demonstrate the flexibility of our framework, we additionally propose a robust version of our model by introducing an inverse-gamma-distributed scale variable, $s$, and the model
\begin{align}
p(s) &= \IGPDF(s ; \lambda_0/2, \lambda_0/2) \label{RobPriorS} \\
p(C \given s) &= \mathcal{MN}(C;C_0, I_d, s V_0)), \label{RobPriorC} \\
p(x_0 \given s) &= \NPDF(x_0; \mu_0, s P_0), \label{RobPriorX} \\
p_\theta(x_k \given x_{k-1}, s) &= \NPDF(x_k; f_\theta(x_{k-1}), s Q_{0}), \label{RobTransitionX} \\
p(y_k \given x_k, C, s) &= \NPDF(y_k; C x_k, s R_0), \label{RobObservationMod}
\end{align}
Note that in this model only the initial noise covariances $Q_0$ and $R_0$ need to be specified, in contrast to the model in \eqref{PriorC}--\eqref{ObservationMod}. By marginalizing out the scale variable $s$ in the multivariate normal distributions we obtain multivariate $t$ distributions (e.g., $p(x_0) = \int \NPDF(x_0; \mu_0, s P_0) \IGPDF(s; \lambda_0/2, \lambda_0/2) \md s = \TPDF(x_0 ; \mu_0, P_0, \lambda_0)$, see \citet{Bishop2006}). This technique has previously been used for robust versions of the Kalman filter \citep{giron1994bayesian,basu1994bayesian,roth2017robust,roth2013student,tronarp2019student}. We follow the approach of \citet{tronarp2019student} to update $Q_k$ and $R_k$ at every iteration. These updates to the noise covariances lead to robustness in light of model misspecification, as discussed in \citet{tronarp2019student}.



\section{INFERENCE AND ESTIMATION}\label{sec:InfAndEst}
Here we derive the algorithm for performing sequential inference in the model \eqref{PriorC}--\eqref{ObservationMod}. Inference in the robust model \eqref{RobPriorS}--\eqref{RobObservationMod} is largely analogous, but necessary modifications are given in Sec.~\ref{sec:robust}. We first present the optimal inference recursions and then describe our approximate inference scheme.

\subsection{Optimal sequential inference}
We give the optimal inference recursions for our model when $\theta$ is assumed to be fixed, and thus drop $\theta$ for notational clarity (parameter estimation is revisited in Sec.~\ref{sec:param_est}). To define a recursive one-step ahead procedure we assume that we are given the filters $p(x_{k-1}|y_{1:k-1})$ and $p(c|y_{1:k-1})$ at time $k-1$.

\paragraph{Prediction.} Using the model \eqref{PriorC}--\eqref{ObservationMod} we compute the predictive distribution as
\begin{align}\label{eq:OptimalPredictive}
p(x_k | y_{1:k-1}) = \int p(x_{k-1}|y_{1:k-1}) p(x_k | x_{k-1}) \md x_{k-1}.
\end{align}
We note that given $p(x_{k-1}|y_{1:k-1})$ this step is independent of the dictionary.

\paragraph{Update.} Given this predictive distribution of $x_k$, we can now define the update steps of the method. In contrast to the Kalman filter, we have two quantities to update: $x_k$ and $c$. We first define the incremental marginal likelihood as
\begin{align}
 p(y_k | y_{1:k-1}) &= \iint p(y_k | c, x_k) p(x_k | y_{1:k-1}) \nonumber \\
 &\qquad\qquad p(c|y_{1:k-1}) \md x_k \md c.
\end{align}
Next, we define the optimal recursions for updating the dictionary $C$ and coefficients $(x_k)_{k\geq 1}$.\\
\textit{Dictionary Update:} Given $p(y_k|y_{1:k-1})$, we can first update the dictionary as follows
\begin{align}\label{eq:OptimalDictUpdate}
p(c|y_{1:k}) = p(c|y_{1:k-1}) \frac{p(y_k | c, y_{1:k-1})}{p(y_k | y_{1:k-1})}
\end{align}
where
\begin{align}\label{eq:OptimalDictLik}
p(y_k | c, y_{1:k-1}) = \int p(y_k | c, x_k) p(x_k | y_{1:k-1}) \md x_k.
\end{align}
\textit{Coefficient Update:} We also update the coefficients at time $k$ (independent of the dictionary) as:
\begin{align}\label{eq:OptimalCoeffUpdate}
p(x_k | y_{1:k}) = p(x_k | y_{1:k-1}) \frac{p(y_k | x_k, y_{1:k-1})}{p(y_k | y_{1:k-1})}
\end{align}
where
\begin{align}\label{eq:OptimalCoeffLik}
p(y_k | x_k, y_{1:k-1}) = \int p(y_k | x_k, c) p(c | y_{1:k-1}) \md c.
\end{align}
Unfortunately, these exact recursions are intractable. In the next section, we make these steps tractable by introducing approximations and obtain an efficient and explicit inference algorithm.

\subsection{Approximate sequential inference}
We start by assuming a special structure on the model. First, we note that the matrix-Gaussian prior in \eqref{PriorC} can be written as $p(c) = \NPDF(c; c_0, V_0 \otimes I_d)$. The Kronecker structure in the covariance will be key to obtain an approximate and tractable posterior distribution with the same covariance structure. To describe our inference scheme we assume that we are given ${p}(c|y_{1:{k-1}}) = \NPDF(c;c_{k-1}, V_{k-1} \otimes I_d)$ and ${p}(x_{k-1}|y_{1:k-1}) = \NPDF(x_{k-1}; \mu_{k-1},P_{k-1})$. Departing from these two distributions it is not possible to exactly update $p(c|y_{1:k})$ and ${p}(x_{k}|y_{1:k})$. As we introduce several approximations we will denote approximate densities with the symbol $\tilde{p}(\cdot)$ instead of $p(\cdot)$ to indicate that the distribution is not exact.

\subsubsection{Prediction}
In the prediction step, we need to compute \eqref{eq:OptimalPredictive}. This is analytically tractable for $f_\theta(x) = Ax$. More specifically, when $f_\theta(x) = Ax$, given ${p}(x_{k-1}|y_{1:k-1}) = \NPDF(x_{k-1}; \mu_{k-1},P_{k-1})$, we obtain ${p}(x_k | y_{1:k-1}) = \NPDF(x_k; \bar{\mu}_k, \bar{P}_k)$ where $\bar{\mu}_k = A\mu_{k-1}$ and $\bar{P}_k = A P_{k-1} A^\top + Q_k$. However, if $f_\theta(x)$ is a nonlinear function, no solution exists and the integral in \eqref{eq:OptimalPredictive} is intractable. In this case, we can use the well-known extended Kalman update (EKF). This update is based on the local linearization of the transition model \citep{mclean1962optimal,anderson1979optimal}, which gives  $\tilde{p}(x_k|y_{1:k-1}) = \NPDF(x_k;\bar{\mu}_k,\bar{P}_k)$ with $\bar{\mu}_k = f_\theta(\mu_{k-1})$ and $\bar{P}_k = F_k P_{k-1} F_k^\top + Q_k$ where $F_k = \frac{\partial f_\theta(x)}{\partial x} \big|_{x = \bar{\mu}_{k-1}}$ is a Jacobian matrix associated with $f_\theta$. The unscented Kalman filter of \citet{julier1997new} can also be used in this step when it is not possible to compute $F_k$ or when $f_{\theta}$ is highly nonlinear. However, since $f_{\theta}$ is a modelling choice (analogous to choosing a kernel function in Gaussian Processes) this scenario is unlikely in practice and the EKF will generally suffice.

\subsubsection{Update}
For the update step, we are interested in updating both $x_k$ and $C$. Given the approximate predictive distribution $\tilde{p}(x_k | y_{1:k-1})$, we would like to obtain $\tilde{p}(c|y_{1:k})$ and $\tilde{p}(x_k | y_{1:k})$. We first describe the update rule for the dictionary $C$, then derive the approximate posterior of $x_k$. Given the prediction, update steps of $C$ and $x_k$ are independent to avoid the repeated use of the data point $y_k$.

\paragraph{Dictionary Update.} To obtain $\tilde{p}(c|y_{1:k})$, we note the integral \eqref{eq:OptimalDictLik} can be computed as
\begin{align}\label{eq:ExactDictionaryLik}
{p}(y_k | c, y_{1:k-1}) = \NPDF(y_k; C \bar{\mu}_k, R_k + C \bar{P}_k C^\top).
\end{align}
This closed form is not helpful to us since this distribution plays the role of the likelihood in  \eqref{eq:OptimalDictUpdate}. Since both the mean and the covariance depend on $C$, the update \eqref{eq:OptimalDictUpdate} is intractable. To solve this problem, we first replace $C \bar{P}_k C^\top \approx C_{k-1} \bar{P}_k C_{k-1}^\top$ in \eqref{eq:ExactDictionaryLik}. This enables a tractable update where the likelihood is of the form $\NPDF(y_k; C \bar{\mu}_k, R_k + C_{k-1} \bar{P}_k C_{k-1}^\top)$. Finally, we choose the Gaussian with a constant diagonal covariance that is closest in terms of $\KL$-divergence and obtain (see, e.g., \citet{probabilisticLMS})
\begin{align}\label{eq:ApproximateDictLikelihood}
\tilde{p}(y_k | c, y_{1:k-1}) = \NPDF(y_k; C \bar{\mu}_k, \eta_k \otimes I_d),
\end{align}
where $\eta_k = {\Tr(R_k + C_{k-1} \bar{P}_k C_{k-1}^\top)}/{d}$. With this approximation the update for the new posterior $\tilde{p}(c|y_{1:k})$ can be computed analytically, given formally in the following proposition based on \citet{akyildiz2019dictionary}.
\begin{prop}\label{prop1}
	Given $\tilde{p}(c | y_{1:k-1}) = \NPDF(c; c_{k-1},  V_{k-1} \otimes 
	I_d )$ and the likelihood $\tilde{p}(y_k | c, y_{1:k-1}) = \NPDF(y_k; 
	C \bar{\mu}_k, \eta_k \otimes I_d)$ the approximate posterior 
	distribution is $\tilde{p}(c|y_{1:k}) = \NPDF(c;c_k,V_k\otimes I_d)$, 
	where $c_k = \vect(C_k)$ and the posterior column-covariance matrix 
	$V_k$ is given by
	\begin{align}
		\label{UpdateV_k}
		V_k = V_{k-1} - \frac{
			V_{k-1} \bar{\mu}_k \bar{\mu}_k^\top V_{k-1}
		}{
			\bar{\mu}_k^\top V_{k-1} \bar{\mu}_k + \eta_k
		} \qquad \text{ for } \quad k \ge 1,
	\end{align}
	and the 
	posterior mean $C_k$ of the dictionary $C$ can be obtained in 
	matrix-form as
	\begin{align}
		\label{FullPosteriorMean}
		C_k = C_{k-1} + \frac{(y_k - C_{k-1} \bar{\mu}_k) 
			\bar{\mu}_k^\top V_{k-1}^{\top}}{\bar{\mu}_k^\top V_{k-1} 
			\bar{\mu}_k + \eta_k} \quad \text{ for } k \ge 1.
	\end{align}
\end{prop}
\begin{proof}
	See Supp.~\ref{sec:appProofProp1}.
\end{proof}
We note the main gain of this result is that we obtain matrix-variate update 
rules for the sufficient statistics of the posterior distribution. This is key 
to an efficient implementation of the method.

\paragraph{Coefficient Update.} To update the posterior density of coefficients, we derive the approximation of $p(y_k | y_{1:k-1},x_k)$ by integrating out $c$, as in \eqref{eq:OptimalCoeffLik}. First, we have the following result.
\begin{prop}\label{propXLik} Given $p(y_k | c, x_k)$ as in 
	\eqref{ObservationMod} and $p(c|y_{1:k-1}) = \NPDF(c; c_{k-1}, V_{k-1} 
	\otimes I_d)$, we obtain
\begin{align}\label{eq:IncrementalMLforxk}
p(y_k | y_{1:k-1},x_k) = \NPDF(y_k; C_{k-1} x_k, R_k + x_k^\top V_{k-1} x_k \otimes I_d).
\end{align}
\end{prop}
\begin{proof}
See Supp.~\ref{sec:appProofProp2}.
\end{proof}
We note that in practice this quantity will be approximate as, e.g., $\tilde{p}(c|y_{1:k-1})$ (and other quantities) will be approximate. However, the likelihood in \eqref{eq:IncrementalMLforxk} with its current form is not amenable to exact inference in \eqref{eq:OptimalCoeffUpdate}, as it contains $x_k$ in both mean and covariance. Therefore, we approximate \eqref{eq:IncrementalMLforxk} by
\begin{align}\label{eq:ApproxCoeffLik}
\tilde{p}(y_k | y_{1:k-1},x_k) = \NPDF(y_k; C_{k-1} x_k, \bar{R}_k),
\end{align}
where $\bar{R}_k = R_k + \bar{\mu}_k^\top V_{k-1}  \bar{\mu}_k \otimes I_d$. With this likelihood, we can obtain the approximate posterior using \eqref{eq:OptimalCoeffUpdate} by an application of the Kalman update \citep{anderson1979optimal}, as $\tilde{p}(x_k | y_{1:k}) = \NPDF(x_k; \mu_k, P_k)$
with
\begin{align}
\mu_k &= \bar{\mu}_k + \bar{P}_k C_{k-1}^\top (C_{k-1} \bar{P}_k C_{k-1}^\top + \bar{R}_k)^{-1} (y_k - C_{k-1} \bar{\mu}_k), \label{eq:coefUpdateMean} \\
P_k &= \bar{P}_k - \bar{P}_k C_{k-1}^\top (C_{k-1} \bar{P}_k C_{k-1}^\top + \bar{R}_k)^{-1} C_{k-1} \bar{P}_k. \label{eq:coefUpdateCov}
\end{align}
Thus we see that the update equations for both the dictionary and the coefficients can be easily implemented by straightforward matrix operations.
\begin{remark}
When $\bar{R}_k$ is diagonal the Woodbury matrix identity \citep{woodbury1950inverting} can be used to accelerate the computation of $(C_{k-1} \bar{P}_k C_{k-1}^{\top} + \bar{R}_k)^{-1}$. We apply this technique in our experiments in Section~\ref{sec:ExperimentPollution}.
\end{remark}

\subsubsection{Inference in the robust model}
\label{sec:robust}
For the robust model in \eqref{RobPriorS}--\eqref{RobObservationMod} inference and estimation proceeds analogously. We provide the full derivation in Supp.~\ref{app:Robust}. As a consequence of the multivariate $t$ distribution the degrees of freedom in the update equations increase by $d$ at every iteration, which we write as $\lambda_k = \lambda_{k-1} + d$. Let $\Delta_{1,k}^{2} = (y_k - C_{k-1}\bar{\mu}_k)^{\top} (C_{k-1} \bar{P}_k C_{k-1}^{\top} + \bar{R}_k)^{-1}(y_k - C_{k-1}\bar{\mu}_k)$ and $\omega_k = (\lambda_{k-1} + \Delta_{1,k}^2)/(\lambda_{k-1} + d)$. Then the reparameterization of the scale variable introduced in \citet{tronarp2019student} results in $s_0 = s$ and $s_k = \omega_k^{-1} s_{k-1}$, as well as the updates $Q_k = \omega_k Q_{k-1}$ and $R_k = \omega_k R_{k-1}$ for the noise covariances. While the mean updates for the coefficients and the dictionary remain unchanged in the robust model, the update of the coefficient covariance $P_k$ and the dictionary column-covariance $V_k$ are affected. The Student's $t$ update for $P_k$ results in multiplication of the right-hand side of \eqref{eq:coefUpdateCov} by $\omega_k$. Now let $\bar{\rho}_k = \bar{\mu}_k^{\top} V_{k-1} \bar{\mu}_k + \eta_k$ and $\Delta_{2,k}^{2} = \|y_k - C_{k-1}\bar{\mu}_k\|^2/\bar{\rho}_k$. Analogously, the right-hand side of \eqref{UpdateV_k} is multiplied by a factor $\varphi_k = {(\lambda_{k-1} + \Delta_{2,k}^2)}/{(\lambda_{k-1} + d)}$. See Supp.~\ref{app:Robust} for full details.

\subsubsection{Parameter estimation}\label{sec:param_est}

To estimate the parameters of $f_{\theta}$ in \eqref{TransitionX}, we need to solve
\begin{align}\label{eq:ParamEstOpt}
\theta^\star \in \argmax_{\theta\in\Theta}  \log p_\theta(y_{1:n}),
\end{align}
using gradient-based schemes \citep{kantas2015particle}. We first present an offline gradient ascent scheme for when the number of observations is relatively small and then introduce a recursive variant that can be used in a streaming setting. 

\paragraph{Iterative estimation.} When the number of data points is limited, it is possible to employ an iterative procedure using multiple passes over data by implementing
\begin{align}\label{eq:offlineSGD}
\theta_i = \theta_{i-1} + \gamma \nabla \log \tilde{p}_\theta(y_{1:n}) \Big|_{\theta = \theta_{i-1}},
\end{align}
at the $i$'th iteration. We refer to this approach as \emph{iterative} PSMF. Since computing $\nabla \log p_\theta(y_{1:n})$ is not possible due to the intractability, we propose to use an approximation $\nabla \log \tilde{p}_\theta(y_{1:n}) = \sum_{k=1}^n \nabla \log \tilde{p}_\theta(y_k | y_{1:k-1})$ that can be computed during forward filtering and removes the need to store all gradients. We remark that it is possible to obtain two approximations of the incremental marginal likelihood $p_\theta(y_k | y_{1:k-1})$ by either integrating out $c$ in \eqref{eq:ApproximateDictLikelihood} or $x_k$ in \eqref{eq:ApproxCoeffLik}. However, the resulting quantities are closely related and we choose the former path for computational reasons. We refer to Sec.~\ref{sec:app:Gradients} for the derivation of the approximate log-marginal likelihood $\log \tilde{p}_\theta(y_k | y_{1:k-1})$.

\paragraph{Recursive estimation.} For long sequences, it is inefficient to perform \eqref{eq:offlineSGD}. Instead, the parameter can be updated online during filtering by fixing $\theta = \theta_{k-1}$ and updating
\begin{align}\label{eq:onlineSGD}
\theta_k = \theta_{k-1} + \gamma \nabla \log \tilde{p}_\theta(y_k | y_{1:k-1})\Big|_{\theta = \theta_{k-1}}.
\end{align}
We call this approach \emph{recursive} PSMF. This is an approximate recursive MLE procedure for SSMs \citep{kantas2015particle}. This procedure has guarantees for finite-state space HMMs, but its convergence for general SSMs is an open problem \citep{kantas2015particle}. The description of iterative and recursive PSMF is given in Algorithm~\ref{algDF}.

\begin{remark} The gradient steps in \eqref{eq:offlineSGD} and \eqref{eq:onlineSGD} can be replaced by modern optimizers to improve convergence, such as Adam \citep{kingma2015adam}. We take advantage of this in Section~\ref{sec:ExperimentSynthetic} below.
\end{remark}

\begin{algorithm}[!tb]
	\small
\begin{algorithmic}[1]
\caption{Iterative and recursive PSMF}\label{algDF}
\State Initialize $\gamma$, $\theta_0$, $C_0$, $V_0$, $\mu_0$, $P_0$, 
$(Q)_{k\geq 1}$, $(R)_{k\geq 1}$.
\For{$i \geq 1$} \Comment{\textcolor{gray}{iterative version}}
\State $C_0 = C_n$, $\mu_0 = \mu_n$, $P_0 = P_n$, $V_0 = V_n$.
\For{$1 \leq k \leq n$}
\State Compute predictive mean of $x_k$:

\hskip\algorithmicindent\hskip\algorithmicindent $\bar{\mu}_k = f_{\theta_{i-1}}(\mu_{k-1})$ or $\bar{\mu}_k = f_{\theta_{k-1}}(\mu_{k-1})$
\State Compute predictive covariance of $x_k$:

\hskip\algorithmicindent $\bar{P}_k = F_k P_{k-1} F_k^\top + Q_k$,  with  $F_k = \frac{\partial f(x)}{\partial x} \big|_{x = \bar{\mu}_{k-1}}$
\State Update dictionary mean $C_k$ using (\ref{FullPosteriorMean})
\State Update dictionary covariance $V_k$ with (\ref{UpdateV_k})
\State Update coefficient mean $\mu_k$ using (\ref{eq:coefUpdateMean}) 
\State Update coefficient covariance $P_k$ with (\ref{eq:coefUpdateCov})
\State Update parameters with (\ref{eq:onlineSGD}) \Comment{\textcolor{gray}{recursive version}}
\EndFor
\State Update parameters with (\ref{eq:offlineSGD}) \Comment{\textcolor{gray}{iterative version}}
\EndFor
\end{algorithmic}
\end{algorithm}

\subsubsection{Approximating the marginal likelihood}\label{sec:app:Gradients}
Consider the likelihood \eqref{eq:ApproximateDictLikelihood} which equals $\tilde{p}_\theta(y_k|y_{1:k-1},c) = \NPDF(y_k; C f_\theta(\mu_{k-1}), \eta_k \otimes I_d)$. Given $\tilde{p}(c | y_{1:k-1}) = \NPDF(c; c_{k-1}, V_{k-1} \otimes I_d)$, the negative log-likelihood is given by (see Supp.~\ref{sec:app:negLogLik})
\begin{align}
-\log \tilde{p}_\theta(y_k | y_{1:k-1}) &\stackrel{c}{=} \frac{d}{2} {\log} \left(\|f_\theta(\mu_{k-1})\|^2_{V_{k-1}} + \eta_k\right) \nonumber \\
&+ \frac{1}{2} \frac{\|y_k - C_{k-1} f_\theta(\mu_{k-1})\|^2}{\eta_k + \|f_\theta(\mu_{k-1})\|_{V_{k-1}}^2}\label{eq:marLogLik}
\end{align}
where $\stackrel{c}{=}$ denotes equality up to some constants that are independent of $\theta$, hence irrelevant for the optimization. Eq.~\eqref{eq:marLogLik} can be seen as an optimization objective that arises from our model. We can compute the gradients of \eqref{eq:marLogLik} using automatic differentiation for generic coefficient dynamics $f_\theta$.

\section{EXPERIMENTS}\label{sec:Experiments}


\begin{figure}[t]
    \centering
    \includegraphics[width=.45\textwidth,height=5cm]{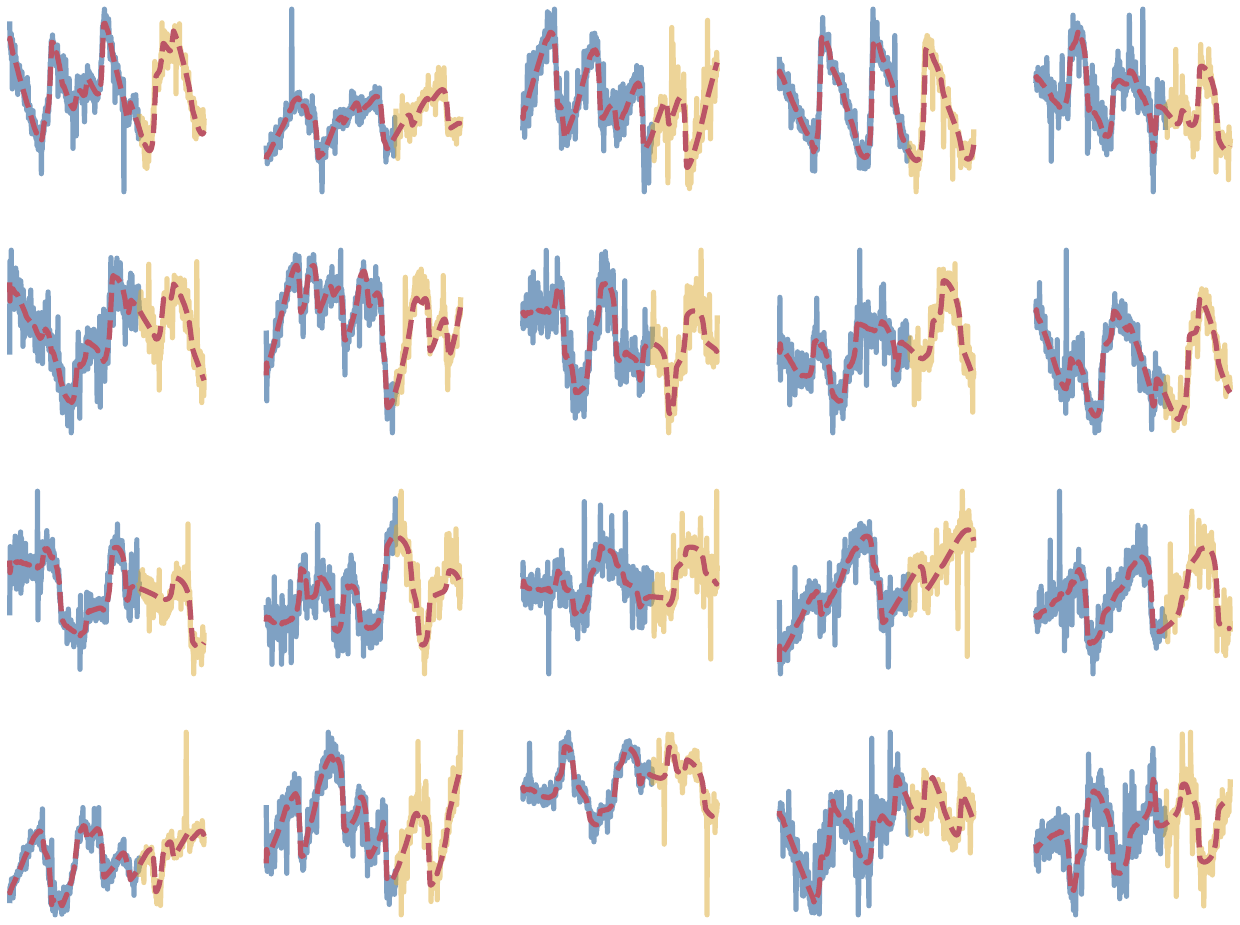}
    \caption{Fitting rPSMF on synthetic data with $t$-distributed noise. Observed time series (\textcolor{SynBlue}{blue}) with unobserved future data (\textcolor{SynYellow}{yellow}) and the reconstruction from the model (\textcolor{SynRed}{red}). \label{fig:synthetic}}
\end{figure}

We evaluate our method in several experiments. In the first two experiments, we show how PSMF can simultaneously learn the dictionary and the parameters of a nonlinear subspace model on synthetic and real data. The third experiment illustrates how our method can be beneficial for change point detection. Finally, our fourth experiment highlights how our method outperforms other MF methods for missing value imputation (modifications to handle missing data in our method are given in Supp.~\ref{app:Missing}). Additionally, in Supp.~\ref{app:Convergence} and Supp.~\ref{app:recursivePSMF} we present experiments on the convergence of our method and recursive parameter estimation, respectively. Code to reproduce our experiments is available in an online repository.\footnote{See: \url{https://github.com/alan-turing-institute/rPSMF}.}

\subsection{A synthetic nonlinear periodic subspace \label{sec:ExperimentSynthetic}}

To demonstrate the ability of the algorithm to learn a dictionary and a structured subspace jointly, we choose $x_0 = \mu_0$ and $P_0 = 0$ and use $x_k = f_\theta(x_{k-1}) = \cos(2\pi\theta k + x_{k-1})$, where $\theta \in \bR_+^r$ and $Q_k = 0$ for all $k\geq 1$. This defines a deterministic subspace with highly periodic structure. We choose $d = 20$ and $r = 6$ and generate the data from the model with $\theta^\star = 10^{-3} \cdot [1,2,3,4,5,6]$. We explore both Gaussian and $t$-distributed measurement noise (the latter using 3 degrees of freedom) and both PMSF and rPSMF. We initialize $C_0$ randomly and draw $\theta_0$ from a uniform distribution on $[0,0.1]^r$. We set $V_0 = v_0 \otimes I_r$ with $v_0 = 0.1$ and use $\lambda_0 = 1.8$ for rPSMF. We furthermore use iterative parameter estimation using the Adam optimizer \citep{kingma2015adam} with standard parameterization, and re-initialize $V_0$, $R_0$, and $Q_0$ at every (outer) iteration (see Supp.~\ref{app:Experiment1} for details). The generated data can be seen in Fig.~\ref{fig:synthetic}. The task is thus to identify the correct subspace structure by estimating $\theta$ as well as to learn the dictionary matrix $C$. 

Fig.~\ref{fig:synthetic} shows a run with $t$-distributed noise fitted by rPSMF. We can see that even though the data exhibits clear outliers the model successfully learns both the underlying generative model and its parameters. Expanded results for PSMF are available in  Supp.~\ref{app:Experiment1}. 

\subsection{Forecasting real-world data using a nonlinear subspace \label{sec:ExperimentBeijing}}

To illustrate the advantage of specifying an appropriate subspace model, we explore weather-related features in a dataset on air quality in Beijing \citep{liang2015assessing}, obtained from the UCI repository \citep{bache2013uci}. To simplify the problem the dataset is sampled at every 100 steps, resulting in $n = 439$ observations and $d = 3$ variables (dew point, temperature, and atmospheric pressure). We compare PSMF using a random walk subspace model, $x_k = f(x_{k-1}) = x_{k-1}$, against a periodic subspace model $x_k = f_{\theta}(x_{k-1}) = \theta_1 \sin(2\pi \theta_2 k + \theta_3 x_{k-1}) + \theta_4 \cos(2\pi \theta_5 k + \theta_6 x_{k-1})$. In both settings we use $r = 1$, run iterative PSMF with 100 iterations, and withhold 20\% of the data for prediction. Fig.~\ref{fig:beijing} illustrates the benefit of an appropriate subspace model on forecasting performance and confirms that PSMF can recover nonlinear subspace dynamics in real-world datasets.

\begin{figure}[t]
\centering
\captionsetup[subfigure]{justification=centering}%
\begin{subfigure}[b]{.45\textwidth}
    \includegraphics[width=\textwidth]{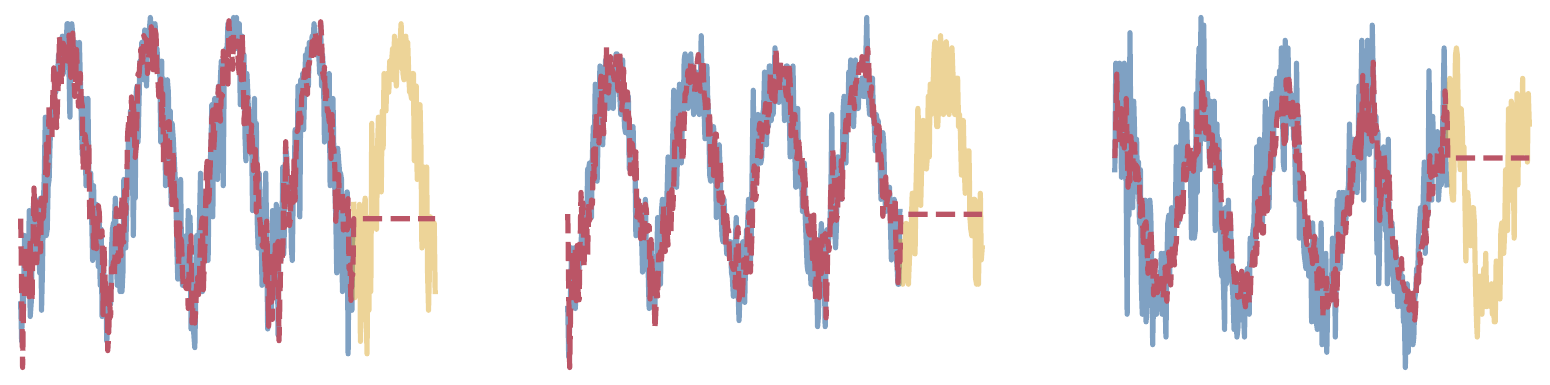}
    \caption{Random walk subspace model. \label{fig:beijing_rw}}
\end{subfigure}
\begin{subfigure}[b]{.45\textwidth}
    \includegraphics[width=\textwidth]{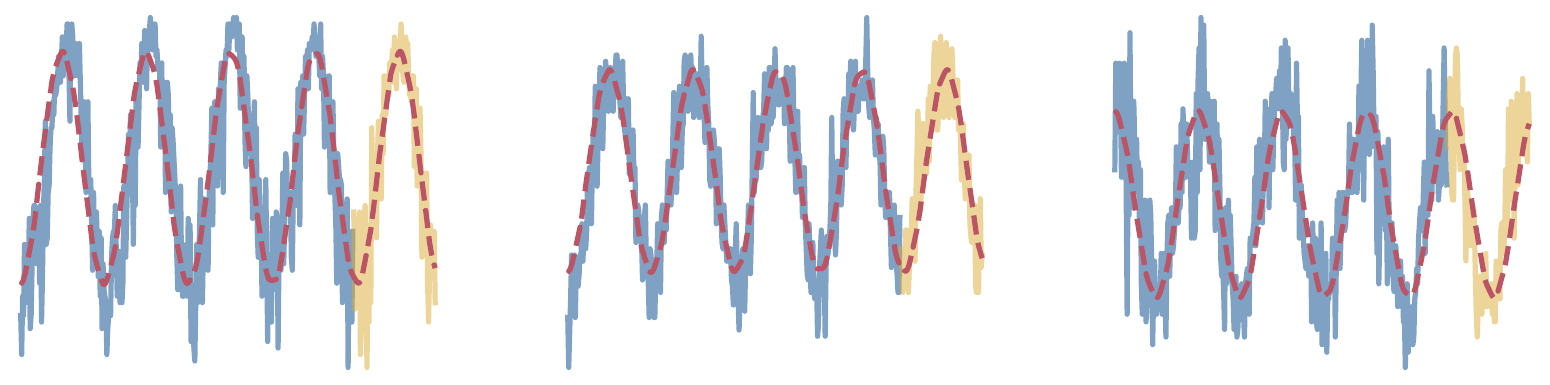}
    \caption{Periodic subspace model. \label{fig:beijing_periodic}}
\end{subfigure}
\caption{Comparison of random walk and periodic subspace models on a time series of weather measurements in Beijing. This shows that with the appropriate subspace model, PSMF correctly identifies the nonlinear dynamics of the data and accurately extrapolates into the future. Colors as in Fig.~\ref{fig:synthetic}. \label{fig:beijing}}
\end{figure}

\subsection{Learning representations for robust multivariate changepoint detection \label{sec:ExperimentChangepoint}}

\begin{table}[tb]
		\caption{Detection accuracy of a changepoint within a window of length $30$ using data and GP features, against the degrees of freedom of the $t$-distributed noise on $\pm 5\%$ of measurements. Averaged over $1000$ different synthetic datasets with $r = 10$.}\label{table:GPcpd} 
		\begin{subtable}{.9\textwidth}
		\footnotesize
		{\begin{tabular}{@{\extracolsep{4pt}}@{\kern\tabcolsep}lccccc}
				\toprule
				&\multicolumn{5}{c}{Degrees of freedom of $t$-contamination}\\
				\cmidrule(lr){2-6}
				& 1.5   & 1.6     & 1.7    & 1.8   & 1.9  \\
				\cmidrule(lr){2-6}
				\multirow{1}{*}{PELT-PSMF}      	&   \textbf{85\%}            &  \textbf{89\%} & \textbf{92\%}      			& \textbf{94\%}           & \textbf{95\%}       \\
				\multirow{1}{*}{PELT-Data}   & 76\%         & 81\%  & 83\%  & 85\% & 85\%
				      \\
				\multirow{1}{*}{MBOCPD}   & 54\%         & 58\%  & 61\%  & 69\% & 72\%
				      \\
				\bottomrule
		\end{tabular}}%
	\end{subtable}%
\end{table}

\begin{table*}[tb]
	\centering
	\caption{Imputation error and runtime on several datasets using 30\% missing values, averaged over 100 random repetitions. An asterisk marks offline methods.}\label{table:Predictive}
	\small
    \centering
\begin{tabular}{lrrrrrrrrrr}
\toprule
 & \multicolumn{5}{c}{Imputation RMSE} & \multicolumn{5}{c}{Runtime (s)} \\\cmidrule(lr){2-6} \cmidrule(lr){7-11}
 & NO$_2$ & PM10 & PM25 & S\&P500 & Gas & NO$_2$ & PM10 & PM25 & S\&P500 & Gas \\ \cmidrule(lr){2-6} \cmidrule(lr){7-11}
PSMF & $\underset{{\scriptscriptstyle (0.13)}}{\textbf{5.72}}$ & $\underset{{\scriptscriptstyle (0.31)}}{\textbf{7.44}}$ & $\underset{{\scriptscriptstyle (0.23)}}{3.55}$ & $\underset{{\scriptscriptstyle \;\;(2.42)}}{11.56}$ & $\underset{{\scriptscriptstyle (1.07)}}{\textbf{6.16}}$ & 2.76 & 2.61 & 1.91 & 9.37 & 96.75\\
rPSMF & $\underset{{\scriptscriptstyle (0.22)}}{5.73}$ & $\underset{{\scriptscriptstyle (0.45)}}{7.54}$ & $\underset{{\scriptscriptstyle (0.21)}}{\textbf{3.50}}$ & $\underset{{\scriptscriptstyle \;\;(1.67)}}{\textbf{10.24}}$ & $\underset{{\scriptscriptstyle (1.51)}}{6.18}$ & 2.93 & 2.03 & 2.02 & 13.06 & 111.89\\
MLE-SMF & $\underset{{\scriptscriptstyle \;\;(0.58)}}{11.17}$ & $\underset{{\scriptscriptstyle (0.31)}}{9.50}$ & $\underset{{\scriptscriptstyle (0.36)}}{4.90}$ & $\underset{{\scriptscriptstyle \;\;(0.83)}}{30.20}$ & $\underset{{\scriptscriptstyle \;\;\;(19.95)}}{111.16}$ & 2.54 & 2.38 & 1.69 & 9.72 & 87.22\\
TMF & $\underset{{\scriptscriptstyle (0.14)}}{7.73}$ & $\underset{{\scriptscriptstyle (0.22)}}{8.08}$ & $\underset{{\scriptscriptstyle (0.31)}}{4.65}$ & $\underset{{\scriptscriptstyle \;\;(0.79)}}{34.90}$ & $\underset{{\scriptscriptstyle \;\;(8.64)}}{74.80}$ & 1.03 & 0.97 & 0.65 & 4.19 & 34.23\\
PMF* & $\underset{{\scriptscriptstyle \;\;(0.06)}}{10.51}$ & $\underset{{\scriptscriptstyle \;\;(0.18)}}{10.49}$ & $\underset{{\scriptscriptstyle (0.18)}}{4.05}$ & $\underset{{\scriptscriptstyle \;\;(1.43)}}{40.69}$ & $\underset{{\scriptscriptstyle \;\;(0.05)}}{23.77}$ & 1.96 & 1.72 & 0.61 & 2.79 & 28.35\\
BPMF* & $\underset{{\scriptscriptstyle (0.20)}}{9.22}$ & $\underset{{\scriptscriptstyle (0.20)}}{8.50}$ & $\underset{{\scriptscriptstyle (0.18)}}{3.68}$ & $\underset{{\scriptscriptstyle \;\;(0.65)}}{27.64}$ & $\underset{{\scriptscriptstyle \;\;(0.28)}}{18.31}$ & 2.89 & 2.71 & 1.61 & 3.68 & 91.30\\
\bottomrule
\end{tabular}
\end{table*}

\begin{table}[tb]
    \centering
	\caption{Average coverage proportion of the missing data by the $2\sigma$ uncertainty bars of the posterior predictive estimates, averaged over 100 repetitions. \label{table:coverage}}
    \footnotesize
    \centering
\begin{tabular}{lccccc}
\toprule
 & NO$_2$ & PM10 & PM25 & S\&P500 & Gas \\
\cmidrule(lr){2-6}
PSMF & 0.76 & 0.76 & \textbf{0.92} & 0.83 & \textbf{0.89} \\
rPSMF & \textbf{0.85} & \textbf{0.89} & 0.87 & \textbf{0.83} & 0.86 \\
MLE-SMF & 0.43 & 0.56 & 0.80 & 0.48 & 0.56 \\
\bottomrule
\end{tabular}
\end{table}

We generate time series with $d = 20$ dimensions where only $3$ exhibit a structural change. In addition to standard Gaussian noise we contaminate 5\% of the entries on average using heavy-tailed $t$-distributed noise with degrees of freedom varying from $1.5$ to $1.9$. To learn the structural changes and be robust against the heavy-tailed noise, we design a smooth subspace model $(x_i(t))_{t\geq 0}$ for $i = 1,\ldots,r$ in continuous-time using a Gaussian process (GP) prior, $x_i(t) \sim \mathcal{GP}(0,\mathsf{k}_\nu(t,t'))$, with Mat{\'e}rn-${3/2}$ kernel with $\nu = 3/2$ \citep{williams2006gaussian}. This particular GP admits a state-space representation amenable to filtering \citep{hartikainen2010kalman} as it can be recast \citep{sarkka2013spatiotemporal} as the stochastic differential equation (SDE): 
\begin{align}\label{eq:SDE_GP}
\frac{\md \mathsf{x}_i(t)}{\md t} = \begin{bmatrix}
0 & 1 \\
-\kappa^2 & - 2\kappa
\end{bmatrix}
\mathsf{x}_i(t) + \begin{bmatrix}
0 \\
1
\end{bmatrix} w_i(t)
\end{align}
where $\mathsf{x}_i(t) = [x_i(t), \md x_i(t) / \md t]$ and $\kappa = \sqrt{2\nu}/{\ell}$. We choose $\sigma^2 = 0.1$ and $\ell = 0.1$ and discretize equation \eqref{eq:SDE_GP} with the step-size $\gamma = 0.001$. We discretize the SDEs for $i = 1,\ldots,r$ and construct a joint state which leads to a linear dynamical system in $2r$ dimensions for which we can run PSMF. The details of the discretization and the corresponding PSMF model are given in Supp.~\ref{app:Experiment2}, along with an illustration of the learned GP features.

We run PSMF with the discretized GP subspace model with $r = 10$. We note that the goal is to obtain a representation that is helpful for changepoint detection. We first employ PELT \citep{killick2012optimal} as an example changepoint detection method directly on the time-series to create a baseline. Then, we estimate a smooth GP subspace with PSMF and run PELT on that subspace (i.e.,~the columns of $X$). We additionally compare against a multivariate implementation of Bayesian online CPD \citep[MBOCPD,][]{adams2007bayesian}. The results in Table~\ref{table:GPcpd} clearly show the improved performance and robustness of PSMF.

\subsection{Missing Value Imputation}
\label{sec:ExperimentPollution}
Finally, we test our method on imputation of missing values in time-series data. We consider data from various domains, including three time series of air pollutants measured across London\footnote{Data collected from \url{https://londonair.org.uk}.}, a gas-sensor dataset by \citet{burgues2018estimation} obtained from the UCI repository \citep{bache2013uci}, and five years of daily closing prices of stocks in the S\&P500 index.\footnote{See: \url{https://www.kaggle.com/camnugent/sandp500}.} The air pollution series contain hourly measurements between 2018-06-01 and 2018-12-01 ($n = 4393$) and consist of NO$_2$ measured at $d = 83$ sites, PM10 from $d = 74$ sites, and PM25 measured at $d = 26$ sites. For the gas sensor dataset (Gas) we have $n = 295,719$ and $d = 19$ and for the S\&P500 dataset we have $n = 1259$ and $d = 505$. The air pollution datasets contain a high number of missing values due to sensor failures and maintenance, and the stock price dataset contains missing values due to stocks being added to the index.

To test the accuracy of imputations, we randomly remove segments of length 
$20$ and thereby construct datasets with $30\%$ missing data. We compare our methods against four baselines. The first is an MLE approach to online probabilistic matrix 
factorization \citep{yildirim2012online,dynamicmatrixfact, fevotte2013non} where we construct an SSM where $C$ is constant, denoted as MLE-SMF. The second is  
temporal matrix factorization (TMF) which is an adaptation of the 
optimisation-based method of \citet{yu2016temporal}. We also add two popular \emph{offline} methods that can only operate on the entire data matrix at once: PMF \citep{mnih2008probabilistic} and BPMF \citep{salakhutdinov2008bayesian}.

We assume the subspace model to be a random walk, $f_\theta(x) = x$, thus avoiding the parameter estimation problem, and we use the final estimates of $C$ and $X$ for data imputation. We formulate TMF with the weight matrix set 
to identity for tractability. We set $r = 10$ for all methods and datasets. For 
PSMF and MLE-SMF we set $R_k := R = \rho \otimes I_r$ with $\rho = 10$, 
$P_0 = I_r$, $Q_k := Q = q \otimes I_r$ with $q = 0.1$. For rPSMF we use $R_0 = R$ and $Q_0 = Q$ and set $\lambda_0 = 1.8$. For PSMF and rPSMF we let 
$V_0 = v_0 \otimes I_r$ where $v_0 = 2$. All methods are run for two iterations (epochs) over the data to limit run-times, and we repeat the experiments $100$ times with different initializations and missing data patterns. In Table~\ref{table:Predictive} we see that PSMF and rPSMF attain lower RMSEs compared to all other methods and that they are competitive in terms of running time. The advantage of our method is especially noticeable on the $\text{NO}_2$, SP500, and Gas datasets.

We can also measure the proportion of missing values that lie within a $2\sigma$ coverage interval of the approximate posterior distribution. Table~\ref{table:coverage} shows how our method improves over the uncertainty quantification of MLE-SMF (the other methods do not provide a posterior distribution). This illustrates the added value of the matrix-variate prior on $C$, as well as our inference scheme. Note that rPSMF obtains a higher coverage percentage than PSMF on three of the datasets, which is due to the sequential updating of the noise covariances. Additional results with 20\% and 40\% missing data are available in Supp.~\ref{app:Experiment3}.

\section{CONCLUSION}\label{sec:Conc}
We have recast the problem of probabilistic dimensionality reduction for time-series as a joint state filtering and parameter estimation problem in a state-space model. Our model is fully probabilistic and we provide a tractable sequential inference algorithm to run the method with linear computational complexity with respect to the number of data points. Our algorithm is purely recursive and can be used in streaming settings. In particular, the batch algorithms that we compare to (such as PMF and BPMF) would incur the full runtime costs when a new sample arrives. This would likely be prohibitive in practice (e.g. BPMF takes 90 seconds for the Gas dataset) and would increase significantly with the dataset size. By contrast, our method only requires incremental computation to process a new sample (e.g. on the order of milliseconds for the Gas dataset).

We have also extended our initial model into a robust version to handle model misspecification and datasets contaminated with outliers. The robust version of our method has been shown to be advantageous in light of model misspecification.

The state-space formulation of the problem opens many directions for future research such as (i) the use of general models for $f_\theta$ or non-Gaussian likelihoods, (ii) the exploration of the use of switching SSMs, and (iii) the integration of more advanced inference techniques such as ensemble Kalman filters or Monte Carlo-based methods for nonlinear and non-Gaussian generalizations of our model.

\acknowledgments{\"O.~D.~A. is supported by the Lloyd’s Register Foundation Data Centric Engineering Programme and EPSRC Programme Grant EP/R034710/1. The work of GvdB is supported in part by EPSRC grant EP/N510129/1 to the Alan Turing Institute. T.~D. acknowledges support from EPSRC EP/T004134/1, UKRI Turing AI Fellowship EP/V02678X/1, and Lloyd’s Register Foundation programme on Data Centric Engineering through the London Air Quality project.}

\bibliography{psmf}
\bibliographystyle{apalike}

\appendix
\numberwithin{equation}{section}

\setcounter{equation}{0}
\setcounter{page}{1}
\setcounter{figure}{0}

\onecolumn
\rule{\textwidth}{4pt}
\vskip\baselineskip
\begin{center}
\Large \bfseries Supplementary Material    
\end{center}%
\rule{\textwidth}{1pt}

\section{PRELIMINARIES}\label{app:prelim}
In this section, we list some linear algebra properties related to Kronecker products, which will be used in proofs.

We denote the Kronecker product $\otimes$. Let $A$ be of dimension ${m\times r}$ and $B$ be of dimension ${r \times n}$; then \cite{harville1997matrix},
\begin{align}
{A}\otimes{B} = \begin{bmatrix} a_{11} {B} & \cdots & a_{1r}{B} \\ \vdots & \ddots & \vdots \\ a_{m1} {B} & \cdots & a_{mr} {B} \end{bmatrix}.
\end{align}
For matrices $A,B$ and $X$, it holds that
\begin{align}\label{generalTrick}
\vect(AXB) = (B^\top \otimes A)\vect(X).
\end{align}
We can particularize this formula for an $r \times 1$ vector $x$ as
\begin{align}\label{vecTrick2}
Ax = \vect(Ax) = (x^\top \otimes I_d) \vect(A).
\end{align}
Kronecker product has the following mixed product property \cite{harville1997matrix}
\begin{align}\label{mixedProduct}
(A \otimes B)(C \otimes D) = (AC) \otimes (BD),
\end{align}
and the inversion property \cite{harville1997matrix}
\begin{align}\label{invKron}
(A\otimes B)^{-1} = A^{-1} \otimes B^{-1}.
\end{align}

\section{PROOF OF PROPOSITION~\ref{prop1}}
\label{sec:appProofProp1}

We adapt the proof in \cite{akyildiz2019dictionary}. We first note that for a 
Gaussian prior $\tilde{p}(c | y_{1:k-1}) = \NPDF(c; c_{k-1},L_{k-1})$ and 
likelihood of the form $p(y_k | y_{1:k-1},c) = \NPDF(y_k; H_k c, G_k)$, we can 
write the posterior analytically $\tilde{p}(c|y_{1:k}) = \NPDF(c; c_k, L_k)$ 
where (see, e.g., \cite{Bishop2006})
\begin{align}
	c_k &= c_{k-1} + L_{k-1} H_k^\top (H_k L_{k-1} H_k^\top + G_k)^{-1} 
	(y_k - H_k c_{k-1}), \label{proof:KalmanMeanUpdate}\\
	L_k &= L_{k-1} - L_{k-1} H_k^\top (H_k L_{k-1} H_k^\top + G_k)^{-1} 
	H_k L_{k-1}. \label{proof:KalmanCovUpdate}
\end{align}

In order to obtain an efficient matrix-variate update rule using this 
vector-form update, we first rewrite the likelihood as
\begin{align}\label{eq:approximateLikelihoodCrewritten}
	\tilde{p}(y_k | c, y_{1:k-1}) &= \NPDF(y_k; H_k c, G_k)
\end{align}
where $H_k = \bar{\mu}_k^\top \otimes I_d$ and $G_k = \eta_k \otimes I_d$. We 
note that, we have $L_0 = V_0 \otimes I_d$ and we assume as an induction 
hypothesis that $L_{k-1} = V_{k-1} \otimes I_d$.  We start by showing that the 
update \eqref{proof:KalmanCovUpdate} can be greatly simplified using the 
special structure we impose. By the mixed product property 
\eqref{mixedProduct} and the inversion property \eqref{invKron} we obtain
\begin{align}
	\left[H_k L_{k-1} H_k^{\top} + G_k\right]^{-1} = 
	\left[(\bar{\mu}_k^{\top} \otimes I_d)(V_{k-1} \otimes 
		I_d)(\bar{\mu}_k \otimes I_d) + \eta_k \otimes I_d\right]^{-1} 
	= (\bar{\mu}_k^{\top} V_{k-1} \bar{\mu}_k + \eta_k)^{-1} \otimes I_d
\end{align}
and therefore,
\begin{align}
	L_{k} =&\,\, (V_{k-1}\otimes I_d) - (V_{k-1}\bar{\mu}_k \otimes I_d) 
	\times ((\bar{\mu}_k^\top V_{k-1} \bar{\mu}_k + \eta_k)^{-1} \otimes 
	I_d) \times (\bar{\mu}_k^\top V_{k-1} \otimes I_d).
\end{align}
One more use of the mixed product property \eqref{mixedProduct} yields
\begin{align}
	L_{k} = \left( V_{k-1} - \frac{V_{k-1} \bar{\mu}_k \bar{\mu}_k^\top 
			V_{k-1}}{\bar{\mu}_k^\top V_{k-1} \bar{\mu}_k + 
			\eta_k} \right) \otimes I_d.
\end{align}
Thus, we have $L_k = V_k \otimes I_d$ where,
\begin{align}
	\label{VkUpdate}
	V_k = V_{k-1} - \frac{V_{k-1} \bar{\mu}_k \bar{\mu}_k^\top 
		V_{k-1}}{\bar{\mu}_k^\top V_{k-1} \bar{\mu}_k + \eta_k}.
\end{align}
We have shown that the sequence $(L_k)_{k\geq 1}$ preserves the Kronecker 
structure. Next, we substitute $L_{k-1} = V_{k-1} \otimes I_d$, $H_k = 
\bar{\mu}_k^\top \otimes I_d$ and $G_k = \eta_k \otimes I_d$ into 
\eqref{proof:KalmanMeanUpdate} and we obtain
\begin{align}
	c_{k} =\,\,&c_{k-1} + (V_{k-1} \otimes I_d) (\bar{\mu}_k \otimes I_d) 
	\times \left( (\bar{\mu}_k^\top V_{k-1} \bar{\mu}_k + \eta_k)^{-1}  
		\otimes I_d \right) \times (y_k - (\bar{\mu}_k^\top \otimes 
	I_d) c_{k-1}).
\end{align}
The use of the mixed product property \eqref{mixedProduct} leaves us with
\begin{align}
	c_{k} =\,\,&c_{k-1} + (V_{k-1} \bar{\mu}_k \otimes I_d) 
	\left((\bar{\mu}_k^\top V_{k-1} \bar{\mu}_k + \eta_k) \otimes I_d 
	\right)^{-1} \times (y_k - (\bar{\mu}_k^\top \otimes I_d) c_{k-1}).
\end{align}
Using \eqref{invKron} and again \eqref{mixedProduct} yields
\begin{align}
	\label{eqck}
	c_{k} = c_{k-1} + &\left[\frac{V_{k-1} \bar{\mu}_k}{\bar{\mu}_k^\top 
			V_{k-1} \bar{\mu}_k + \eta_k} \otimes I_d \right] 
	\times (y_k - (\bar{\mu}_k^\top \otimes I_d) c_{k-1}).
\end{align}
Using \eqref{vecTrick2}, we get
\begin{equation}
	c_{k} = c_{k-1} +\left[
	\frac{V_{k-1} \bar{\mu}_k}{\bar{\mu}_k^\top V_{k-1} \bar{\mu}_k + 
		\eta_k} \otimes I_d \right] (y_k - C_{k-1} \bar{\mu}_k).
\end{equation}
We now note that $(y_k - C_{k-1} \bar{\mu}_k)$ and $\frac{V_{k-1} 
	\bar{\mu}_k}{\bar{\mu}_k^\top V_{k-1} \bar{\mu}_k + \eta_k}$ are 
vectors.  Hence, rewriting the above expression as
\begin{align}
	\label{eqck1}
	c_{k} = c_{k-1} &+\left[
		\vect\left(\frac{V_{k-1} \bar{\mu}_k}{\bar{\mu}_k^\top V_{k-1} 
				\bar{\mu}_k + \eta_k}\right) \otimes I_d  
	\right] \times \vect(y_k - C_{k-1} \bar{\mu}_k),
\end{align}
we can apply \eqref{vecTrick2} and obtain
\begin{equation}
	c_{k} = c_{k-1} +\vect\left(\frac{(y_k - C_{k-1} \bar{\mu}_k) 
			\bar{\mu}_k^\top V_{k-1}^\top}{\bar{\mu}_k^\top 
			V_{k-1} \bar{\mu}_k + \eta_k}\right).
\end{equation}
Hence up to a reshaping operation, we have the update rule 
\eqref{FullPosteriorMean} and conclude the proof. \hfill $\qed$

\section{PROOF OF PROPOSITION~\ref{propXLik}}
\label{sec:appProofProp2}
Recall that we have a posterior of the form at time $k-1$
\begin{align}
	p(c|y_{1:k-1}) = \NPDF(c; c_{k-1}, V_{k-1} \otimes I_d),
\end{align}
and we are given the likelihood
\begin{align}
	p(y_k | c,x_k) = \NPDF(y_k; (x_k\otimes I_d) c, R_k).
\end{align}
We are interested in computing
\begin{align}
	p(y_k | y_{1:k-1},x_k) &= \int p(c|y_{1:k-1}) p(y_k | c,x_k) \md c.
\end{align}
This integral is analytically tractable since both distributions are Gaussian 
and it is given by \cite{Bishop2006}
\begin{align}
	p(y_k | y_{1:k-1},x_k) = \NPDF(y_k; (x_k^\top \otimes I_d) c_k, R_k + 
	(x_k^\top \otimes I_d) (V_{k-1} \otimes I_d) (x_k \otimes I_d)).
\end{align}
Using the mixed product property \eqref{mixedProduct}, one obtains
\begin{align}
	p(y_k | y_{1:k-1},x_k) = \NPDF(y_k; C_{k-1} x_k, R_k + x_k^\top 
	V_{k-1} x_k \otimes I_d).
\end{align}

\section{DERIVATION OF THE NEGATIVE LOG-LIKELIHOOD}\label{sec:app:negLogLik}
We obtain the marginal likelihood as
\begin{align}
\tilde{p}_\theta(y_k | y_{1:k-1}) &= \int \tilde{p}(y_k | y_{1:k-1},c) \tilde{p}(c | y_{1:k-1}) \md c \\
&= \NPDF(y_k ; C \bar{\mu}_k, \eta_k \otimes I_d) \NPDF(c ; c_{k-1}, V_{k-1} \otimes I_d) \\
&= \NPDF(y_k ; (\bar{\mu}_k^{\top} \otimes I_d)c, \eta_k \otimes I_d) \NPDF(c ; c_{k-1}, V_{k-1} \otimes I_d) \\
&= \NPDF(y_k ; (\bar{\mu}_k^{\top} \otimes I_d)c_{k-1}, (\bar{\mu}_k^\top V_{k-1} \bar{\mu}_k + \eta_k) \otimes I_d) \\
&= \NPDF\left(y_k; C_{k-1} f_\theta(\mu_{k-1}), \left(\|f_\theta(\mu_{k-1})\|^2_{V_{k-1}} + \eta_k\right) \otimes I_d \right).
\end{align}
where in the last line we have used the fact that $\bar{\mu}_k = f_{\theta}(\mu_{k-1})$ and properties from Supp.~\ref{app:prelim}. It is then straightforward to show that
\begin{align}
    &- \log \tilde{p}_{\theta}(y_k \given y_{1:k-1}) = -\log \left[ (2\pi)^{-d/2} \cdot | (\|f_\theta(\mu_{k-1})\|^2_{V_{k-1}} + \eta_k) \otimes I_d|^{-1/2} \right. \\
    &\quad \left. \cdot \exp\left( -\tfrac{1}{2}(y_k - C_{k-1}f_{\theta}(\mu_{k-1}))^{\top}\left( \|f_\theta(\mu_{k-1})\|^2_{V_{k-1}} + \eta_k) \otimes I_d \right)^{-1} (y_k - C_{k-1} f_{\theta}(\mu_{k-1}) \right) \right]
\end{align}
which simplifies to
\begin{align}
    - \log \tilde{p}_{\theta}(y_k \given y_{1:k-1}) &= \frac{d}{2} \log (2\pi) + \frac{d}{2} \log(\|f_\theta(\mu_{k-1})\|^2_{V_{k-1}} + \eta_k) + \tfrac{1}{2} \frac{ \| y_k - C_{k-1} f_{\theta}(\mu_{k-1}) \|^2}{ \|f_\theta(\mu_{k-1})\|^2_{V_{k-1}} + \eta_k }.
\end{align}

\section{THE PROBABILISTIC MODEL TO HANDLE MISSING DATA}\label{app:Missing}

To obtain update rules that can explicitly handle missing data, we only need to modify the likelihood. When we receive an observation vector with missing entries, we model it as $z_k = m_k \odot y_k$ where $m_k \in \{0,1\}^d$ is a mask vector that contains zeros for missing entries and ones otherwise. We note that $z_k = M_k y_k$ where $M_k = \diag(m_k)$, which results in the likelihood $p(z_k  | c, x_k) = \NPDF(z_k ; M_k C x_k, M_k R_k M_k^\top)$. The update rules for PSMF and the robust model, rPSMF, can be easily re-derived using this likelihood and are essentially identical to Algorithm~\ref{algDF} with masks. Here we discuss the case of PSMF with missing values, rPSMF with missing values is discussed in Supp.~\ref{app:Robust}.

We define the probabilistic model with missing data as
\begin{align}
p(C) &= \mathcal{MN}(C;C_0,I_d,V_0),  \\
p(x_0) &= \NPDF(x_0; \mu_0, P_0), \\
p_\theta(x_k | x_{k-1}) &= \NPDF(x_k; f_\theta(x_{k-1}), Q_k), \\
p(z_k | x_k, C) &= \NPDF(z_k; M_k C x_k, M_k R_k M_k^\top).
\end{align}
This model can explicitly handle the missing data when $(M_k)_{k\geq 1}$ (the missing data patterns) are given. The update rules for this model are defined using masks and are similar to the full data case. In what follows, we derive the update rules for this model by explicitly handling the masks and placing them into our updates formally. For the missing-data case, however, we need a minor approximation in the covariance update rule in order to keep the method efficient. 
Assume that we are given $\tilde{p}(c | z_{1:k-1}) = \NPDF(c; c_{k-1},  V_{k-1} \otimes I_d )$ and the likelihood
\begin{align}
\tilde{p}(z_k  | c, z_{1:k-1}) = \NPDF(z_k ; M_k C \bar{\mu}_k, \eta_k \otimes I_d)
\end{align}
where
\begin{align}
\eta_k = \frac{\Tr(M_k R_k M_k^\top + M_k C_{k-1} \bar{P}_k C_{k-1}^\top M_k^\top)}{m}.
\end{align}
In the sequel, we derive the update rules corresponding to the our method with missing data. The derivation relies on the proof of Prop.~\ref{prop1}. We note that using \eqref{generalTrick}, we can obtain the likelihood
\begin{align}
\tilde{p}(z_k | c, z_{1:k-1}) = \NPDF(z_k; H_k c, \eta_k \otimes I_d)
\end{align}
where $c = \vect(C)$ and $H_k = \bar{\mu}_k^\top \otimes M_k$. Deriving the posterior in the same way as in the proof of Prop.~\ref{prop1}, and using the approximation $\bar{\mu}_k^{\top} V_{k-1} \bar{\mu}_k \otimes M_k \approx \bar{\mu}_k^{\top} V_{k-1} \bar{\mu}_k \otimes I_d$, leaves us with the covariance update in the form
\begin{align}
P_k = V_{k-1} \otimes I_d - \frac{V_{k-1} \bar{\mu}_k \bar{\mu}_k^\top V_{k-1}}{\bar{\mu}_k^\top V_{k-1} \bar{\mu}_k + \eta_k} \otimes M_k.
\end{align}
Unlike the previous case, this covariance does not simplify to a form $P_k = V_k \otimes I_d$ easily. For this reason, we approximate it as
\begin{align}
P_k \approx V_k \otimes I_d,
\end{align}
where $V_k$ is in the same form of missing-data free updates. To update the mean, we proceed in a similar way as in the proof of Prop.~\ref{prop1} as well. Straightforward calculations lead to the update
\begin{align}\label{FullPosteriorMeanMissing}
C_k = C_{k-1} + \frac{(z_k - M_k C_{k-1} \bar{\mu}_k) \bar{\mu}_k^\top V_{k-1}}{\bar{\mu}_k^\top V_{k-1} \bar{\mu}_k + \eta_k}, \quad \mbox{ for $k \ge 1$.}
\end{align}
To update $x_k$, once we fix $C_{k-1}$, everything straightforwardly follows by replacing $C_{k-1}$ by $M_k C_{k-1}$ in the update rules for $(x_k)_{k\geq 1}$. Finally, the negative log-likelihood  $\tilde{p}_{\theta}(z_k | z_{1:k-1})$ can be derived similarly to the non-missing case in Sec.~\ref{sec:app:Gradients}, and equals
\begin{equation}
    -\log \tilde{p}_{\theta}(z_k | z_{1:k-1}) \stackrel{c}{=} \tfrac{1}{2} \sum_{j=1}^d \log u_{jk} + \tfrac{1}{2} (z_k - M_k C_{k-1} f_{\theta}(\mu_{k-1}))^{\top} U_k^{-1} (z_k - M_k  C_{k-1}f_{\theta}(\mu_{k-1})),
\end{equation}
where $\stackrel{c}{=}$ denotes equality up to constants that do not depend on $\theta$ and $U_k = \|f_{\theta}(\mu_{k-1})\|_{V_{k-1}}^2 \otimes M_k + \eta_k \otimes I_d$ is a $d$-dimensional diagonal matrix with elements $u_{jk}$ for $j=1,\ldots,d$.

\section{THE ROBUST MODEL}
\label{app:Robust}

Recall that the model definitions for robust PSMF are as follows
\begin{align}
p(s) &= \IGPDF(s ; \lambda_0/2, \lambda_0/2) \\
p(C \given s) &= \mathcal{MN}(C;C_0, I_d, s V_0)), \\
p(x_0 \given s) &= \NPDF(x_0; \mu_0, s P_0),\\
p_\theta(x_k \given x_{k-1}, s) &= \NPDF(x_k; f_\theta(x_{k-1}), s Q_{0}),  \\
p(y_k \given x_k, C, s) &= \NPDF(y_k; C x_k, s R_0), 
\end{align}

Before we present the derivation, we recall the following definitions.

\begin{defn}[Inverse-Gamma Distribution]
The inverse-gamma distribution is given by
\begin{equation}
    \IGPDF(s ; \alpha, \beta) = \frac{\beta^{\alpha}}{\Gamma(\alpha)} \left( \frac{1}{s} \right)^{\alpha + 1} \exp \left( -\beta / s \right)
\end{equation}
for $\alpha, \beta > 0$, and with $\Gamma(\cdot)$ the Gamma function.
\end{defn}
\begin{defn}[Multivariate $t$ Distribution]
For $y \in \mathbb{R}^d$ the multivariate $t$ distribution with $\lambda$ degrees of freedom is
\begin{equation}
    \TPDF(y ; \mu, \Sigma, \lambda) = \frac{1}{(\pi \lambda)^{d/2} |\Sigma|^{1/2}} \frac{\Gamma((\lambda + d)/2)}{\Gamma(\lambda/2)} \left( 1 + \frac{\Delta^2}{\lambda} \right)^{-(\lambda + d)/2}
\end{equation}
where $\Delta^2 = (y - \mu)^{\top} \Sigma^{-1} (y - \mu)$.
\end{defn}

Since we again assume the model to be Markovian, we extend the conditional independence and Markov properties \cite[see, e.g.,][]{sarkka2013bayesian} to the case with a scale variable
\begin{propty}[Conditional independence]
\label{propty:rob_cond_ind}
The measurement $y_k$ given the coefficient $x_k$ and scale variable $s$, is conditionally independent of past measurements and coefficients
\begin{equation}
    p(y_k \given x_{1:k}, y_{1:k-1}, s) = p(y_k \given x_k, s).
\end{equation}
\end{propty}
\begin{propty}[Markov property of coefficients]
\label{propty:rob_markov}
When conditioning on $s$ the coefficients $x_k$ form a Markov sequence, such that
\begin{equation}
    p(x_k \given x_{1:k-1}, y_{1:k-1}, s) = p(x_k \given x_{k-1}, s).
\end{equation}
\end{propty}

We also present the following lemma's used in the derivation.

\begin{lem}\label{lem:normal_ig_posterior}
For $y \in \mathbb{R}^d$ with $p(y \given s) = \NPDF(y ; \mu, \Sigma)$ and $p(s) = \IGPDF(s ; \alpha, \beta)$ we have
\begin{align}
p(y) &= \frac{1}{(2 \pi \beta)^{d/2} |\Sigma|^{1/2}} \frac{\Gamma(\alpha + d/2)}{\Gamma(\alpha)} \left( 1 + \frac{\Delta^2}{2\beta} \right)^{-(\alpha + d/2)} \\
p(s | y) &= \IGPDF(s ; \alpha + d/2, \beta + \tfrac{1}{2}\Delta^2) .
\end{align}
In particular, if $\alpha = \beta = \lambda / 2$ then $p(y) = \TPDF(y ; \mu, \Sigma, \lambda)$.
\end{lem}

\begin{lem}\label{lem:igtransform}
If $p(s) = \IGPDF(s ; \alpha, \beta)$ and $\omega = \beta/\alpha$, then $\omega \cdot p(\omega s) = \IGPDF(s ; \alpha, \alpha)$.
\end{lem}
\begin{proof}
\begin{equation}
\frac{\beta}{\alpha} p\left(\frac{\beta}{\alpha} s\right) = \frac{\beta}{\alpha} \frac{\beta^{\alpha}}{\Gamma(\alpha)} \left( \frac{\alpha}{\beta s} \right)^{\alpha + 1} \exp \left( - \frac{\beta \alpha}{\beta s} \right) = \frac{\alpha^{\alpha}}{\Gamma(\alpha)} \left( \frac{1}{s} \right)^{\alpha + 1} \exp\left( -\frac{\alpha}{s} \right) = \IGPDF(s ; \alpha, \alpha).
\end{equation}
\end{proof}

\begin{lem}\label{lem:partitioned_t}
For a partitioned random variable $y = [y_a, y_b]^{\top}$ with $y_a \in \mathbb{R}^{d_a}$ and $y_b \in \mathbb{R}^{d_b}$ that follows a multivariate $t$ distribution given by
\begin{equation}
p(y) = p(y_a, y_b) = 
\TPDF\left(
    \begin{bmatrix}
    y_a \\
    y_b
    \end{bmatrix}
    ;
    \begin{bmatrix}
    \mu_a \\
    \mu_b
    \end{bmatrix}
    ,
    \begin{bmatrix}
    \Sigma_{aa} & \Sigma_{ab} \\
    \Sigma_{ab}^{\top} & \Sigma_{bb}
    \end{bmatrix}
    ,
    \lambda
    \right),
 \end{equation}
the marginal and conditional densities are given by
\begin{align}
    p(y_b) &= \TPDF(y_b ; \mu_b, \Sigma_{bb}, \lambda) \\
    p(y_a \given y_b) &= \TPDF(y_a ; \mu_{a|b}, \Sigma_{a|b}, \lambda_{a|b}),
\end{align}
with
\begin{align}
    \lambda_{a|b} &= \lambda + d_b \\
    \mu_{a|b} &= \mu_a + \Sigma_{ab} \Sigma_{bb}^{-1} (y_b - \mu_b) \\
    \Sigma_{a|b} &= \frac{\lambda + (y_b - \mu_b)^{\top}\Sigma_{bb}^{-1} (y_b - \mu_b)}{\lambda + d_b} \left( \Sigma_{aa} - \Sigma_{ab} \Sigma_{bb}^{-1} \Sigma_{ab}^{\top} \right).
\end{align}
\end{lem} 
\begin{proof}
See \cite{roth2012multivariate} for a derivation.
\end{proof}

To derive inference in the robust model, we start from $k = 1$ and show how we perform filtering for an entire iteration. While this makes the description longer, we believe it to be more informative for the reader. We begin with prediction of $x_1$ given no history ($y_{1:0} = \emptyset$). The predictive distribution of $x_1$ is then
\begin{align}
    \tilde{p}(x_1 \given y_{1:0}, s) &= \int p(x_1 \given x_0, s) p(x_0 \given y_{1:0}, s) \md x_0 \\
    \tilde{p}(x_1 \given s) &= \int p(x_1 \given x_0, s) p(x_0 \given s) \md x_0 \\
    &= \int \NPDF(x_1 ; f_{\theta}(x_0), s Q_0) \NPDF(x_0 ; \mu_0, s P_0) \md x_0 \label{eq:coefpredint1} \\
    &= \NPDF(x_1 ; f_{\theta}(\mu_0), s (Q_0 + F_1 P_0 F_1^{\top})),
\end{align}
where $F_1$ is defined as in the main text. Writing $\bar{\mu}_1 = f_{\theta}(\mu_0)$ and $\bar{P}_1 = Q_0 + F_1 P_0 F_1^{\top}$ we get $\tilde{p}(x_1 \given s) = \NPDF(x_1 ; \bar{\mu}_1, s \bar{P}_1)$. Next, we move to the dictionary update. We first have
\begin{align}
    \tilde{p}(y_1 \given c, y_{1:0}, s) &= \int p(y_1 \given c, x_1, s) p(x_1 \given y_{1:0}, s) \md x_1 \\
    \tilde{p}(y_1 \given c, s) &= \int p(y_1 \given c, x_1, s) p(x_1 \given s) \md x_1 \\
    &= \int \NPDF(y_1 ; C x_1, s R_0) \NPDF(x_1 ; \bar{\mu}_1, s \bar{P}_1) \md x_1 \\
    &= \NPDF(y_1 ; C \bar{\mu}_1, s (R_0 + C \bar{P}_1 C^{\top})).
\end{align}
As in PSMF, we use the approximation $C \bar{P}_1 C^{\top} \approx \eta_1 \otimes I_d$ where $\eta_1 = \Tr(R_0 + C_0 \bar{P}_1 C_0^{\top})/d$. We write this as $\tilde{p}(y_1 \given c, s) = \NPDF(y_1 ; H_1 c, s G_1)$ with $H_1 = \bar{\mu}_1^{\top} \otimes I_d$ and $G_1 = \eta_1 \otimes I_d$. We again assume $\tilde{p}(c \given y_{1:0}, s) = \NPDF(c ; c_0, s L_0)$ using $L_0 = V_0 \otimes I_d$, such that
\begin{align}
    \tilde{p}(c, y_1 \given y_{1:0}, s) &= \tilde{p}(y_1 \given c, y_{1:0}, s) \tilde{p}(c \given y_{1:0}, s) \label{eq:robustDictJoint1} \\
    \tilde{p}(c, y_1 \given s) &= \tilde{p}(y_1 \given c, s) \tilde{p}(c \given s) \\
    &= \NPDF(y_1 ; H_1 c, s G_1) \NPDF(c ; c_0, s L_0) \\
    &= \NPDF\left(
    \begin{bmatrix}
    c \\
    y_1
    \end{bmatrix}
    ;
    \begin{bmatrix}
    c_0 \\
    H_1 c_0
    \end{bmatrix}
    ,
    s \begin{bmatrix}
    L_0 & L_0 H_1^{\top} \\
    H_1 L_0 & H_1 L_0 H_1^{\top} + G_1
    \end{bmatrix}
    \right)
\end{align}
Integrating out $s$ in this expression gives
\begin{equation}
    \tilde{p}(c, y_1) = \TPDF\left(
    \begin{bmatrix}
    c \\
    y_1
    \end{bmatrix}
    ;
    \begin{bmatrix}
    c_0 \\
    H_1 c_0
    \end{bmatrix}
    ,
    \begin{bmatrix}
    L_0 & L_0 H_1^{\top} \\
    H_1 L_0 & H_1 L_0 H_1^{\top} + G_1
    \end{bmatrix}
    ,
    \lambda_0
    \right)
\end{equation}
Conditioning on $y_1$ and using Lemma~\ref{lem:partitioned_t} yields $\tilde{p}(c \given y_1) = \TPDF(c ; c_1, L_1, \lambda_0 + d)$ with
\begin{align}
    \label{eq:robustDictUpdate1}
    c_1 &= c_0 + L_0 H_1^{\top} \left[ H_1 L_0 H_1^{\top} + G_1 \right]^{-1} (y_1 - H_1 c_0) \\
    L_1 &= \phi_1 \left( L_0 H_1^{\top} \left[ H_1 L_0 H_1^{\top} + G_1 \right]^{-1} H_1 L_0 \right) \\
    \varphi_1 &= \frac{ \lambda_0 + (y_1 - H_1 c_0)^{\top} \left[ H_1 L_0 H_1^{\top} + G_1 \right]^{-1} (y_1 - H_1 c_0)}{ \lambda_0 + d}.
\end{align}
This is the \textbf{robust PSMF dictionary update}. We see that the mean is updated as in PSMF by comparing to \eqref{proof:KalmanMeanUpdate}, and that the covariance update has an additional multiplicative factor $\varphi_1$. These expressions can be simplified by plugging in the definitions of $L_0$, $H_1$, and $G_1$, as in Supp.~\ref{sec:appProofProp1}. Observe that $\tilde{p}(c \given y_1)$ can no longer be written as an infinite scale mixture with scale variable $s$, as they now differ in degrees of freedom. We will revisit this point below.

For the coefficient update we proceed analogously. First, note that
\begin{align}
    \tilde{p}(y_1 \given x_{0:1}, s) &= \int p(y_1 \given c, x_{0:1}, s) p(c \given y_{1:0}, s) \md c \\
    \tilde{p}(y_1 \given x_1, s) &= \int p(y_1 \given c, x_1, s) p(c \given s) \md c \\
    &= \int \NPDF(y_1 ; (x_1^{\top} \otimes I_d) c, s R_0) \NPDF(c ; c_0, s L_0) \md c \label{eq:y1_given_x1_s} \\
    &= \NPDF(y_1 ; (x_1^{\top} \otimes I_d) c_0, s (R_0 +  x_1^{\top} V_{0} x_1 \otimes I_d))
\end{align}
As in the main text, we use the approximation $x_1^{\top} V_{0} x_1 \approx \bar{\mu}_1^{\top} V_0 \bar{\mu}_1$ and introduce
\begin{equation}
    \bar{R}_0 = R_0 + \bar{\mu}_1^{\top} V_0 \bar{\mu}_1,
\end{equation}
such that $\tilde{p}(y_1 \given x_1, s) = \NPDF(y_1 ; C_0 x_1, s \bar{R}_0)$. We then find the joint distribution between $x_1$ and $y_1$ as follows
\begin{align}
    \tilde{p}(x_1, y_1 \given y_{1:0}, s) &= \tilde{p}(y_1 \given y_{1:0}, x_1, s) \tilde{p}(x_1 \given y_{1:0}, s) \\
    \tilde{p}(x_1, y_1 \given s) &= \tilde{p}(y_1 \given x_1, s) \tilde{p}(x_1 \given s) \\
    &= \NPDF(y_1 ; C_0 x_1, s \bar{R}_0) \NPDF(x_1 ; \bar{\mu}_1, s \bar{P}_1) \\
    &= \NPDF\left(
    \begin{bmatrix}
    x_1 \\
    y_1
    \end{bmatrix}
    ;
    \begin{bmatrix}
    \bar{\mu}_1 \\
    C_0 \bar{\mu}_1
    \end{bmatrix}
    ,
    s \begin{bmatrix}
    \bar{P}_1 & \bar{P}_1 C_0^{\top} \\
    C_0 \bar{P}_1 & C_0 \bar{P}_1 C_0^{\top} + \bar{R}_0
    \end{bmatrix}
    \right). \label{eq:robust_x1y1_joint_normal}
\end{align}
Integrating out $s$ in this expression gives
\begin{equation}
    \tilde{p}(x_1, y_1) = \TPDF\left(
    \begin{bmatrix}
    x_1 \\
    y_1
    \end{bmatrix}
    ;
    \begin{bmatrix}
    \bar{\mu}_1 \\
    C_0 \bar{\mu}_1
    \end{bmatrix}
    ,
    \begin{bmatrix}
    \bar{P}_1 & \bar{P}_1 C_0^{\top} \\
    C_0 \bar{P}_1 & C_0 \bar{P}_1 C_0^{\top} + \bar{R}_0
    \end{bmatrix}
    ,
    \lambda_0
    \right).
\end{equation}
Conditioning on $y_1$ and using Lemma~\ref{lem:partitioned_t} gives $p(x_1 \given y_1) = \TPDF(x_1 ; \mu_1, P_1, \lambda_0 + d)$ with
\begin{align}
    \label{eq:robustCoefUpdate1}
    \mu_1 &= \bar{\mu}_1 + \bar{P}_1 C_0^{\top} \left[ C_0 \bar{P}_1 C_0^{\top} + \bar{R}_0 \right]^{-1} (y_1 - C_0 \bar{\mu}_1) \\
    P_1 &= \omega_1 \left( \bar{P}_1 - \bar{P}_1 C_0^{\top} \left[ C_0 \bar{P}_1 C_0^{\top} + \bar{R}_0 \right]^{-1} C_0 \bar{P}_1 \right) \\
    \omega_1 &= \frac{ \lambda_0 + (y_1 - C_0 \bar{\mu}_1)^{\top} \left[ C_0 \bar{P}_1 C_0^{\top} + \bar{R}_0 \right]^{-1} (y_1 - C_0 \bar{\mu}_1) }{\lambda_0 + d}.
\end{align}
This is the \textbf{robust PSMF coefficient update}. Again we see that the mean update for $\mu_1$ is the same as in vanilla PSMF, while the covariance update has an additional multiplicative factor $\omega_1$. By introducing $\Delta_1^2 = (y_1 - C_0 \bar{\mu}_1)^{\top} \left[ C_0 \bar{P}_1 C_0^{\top} + \bar{R}_0 \right]^{-1} (y_1 - C_0 \bar{\mu}_1)$ we can simplify this factor to $\omega_1 = (\lambda_0 + \Delta_1^2)/(\lambda_0 + d)$.

Finally, we can compute the posterior of the scale variable, $s$, using Bayes' theorem,
\begin{equation}
    \tilde{p}(s \given y_1) = \frac{\tilde{p}(y_1 \given s) p(s)}{\tilde{p}(y_1)}.
\end{equation}
We can obtain $\tilde{p}(y_1 \given s)$ from \eqref{eq:robust_x1y1_joint_normal}, which yields
\begin{align}
    \tilde{p}(y_1 \given s) = \NPDF(y_1 ; C_0 \bar{\mu}_1, s (C_0 \bar{P}_1 C_0^{\top} + \bar{R}_0)).
\end{align}
Integrating out $s$ gives $\tilde{p}(y_1) = \TPDF(y_1 ; C_0 \bar{\mu}_1, C_0 \bar{P}_1 C_0^{\top} + \bar{R}_0, \lambda_0)$. Thus, by Lemma~\ref{lem:normal_ig_posterior} we have
\begin{equation}
    \label{eq:rob_s_given_y1}
    \tilde{p}(s \given y_1) = \IGPDF(s ; (\lambda_0 + d)/2, (\lambda_0 + \Delta_1^2)/2).
\end{equation}

Having now observed $y_1$, we proceed with the next iteration. Note that from the coefficient update we have obtained $p(x_1 \given y_1) = \TPDF(x_1 ; \mu_1, P_1, \lambda_0 + d)$. We can write this as a infinite scale mixture by defining $u \sim \IGPDF(u ; (\lambda_0 + d)/2, (\lambda_0 + d)/2)$ and introducing $p(x_1 \given y_1, u) = \NPDF(x_1; \mu_1, u P_1)$. The model definitions give the coefficient dynamics in terms of $s$, as $p(x_2 \given x_1, s) = \NPDF(x_2 ; f_{\theta}(x_1), sQ_0)$. This can be written in terms of $u$ by a simple change of variables $u = \omega_1^{-1} s$ and using Lemma~\ref{lem:igtransform} and \eqref{eq:rob_s_given_y1}, since
\begin{align}
    \tilde{p}(x_2 \given x_1, y_1) &= \int p(x_2 \given x_1, s) p(s \given y_1) \md s \\
    &= \int \NPDF(x_2 ; f_{\theta}(x_1), s Q_0) \IGPDF(s ; (\lambda_0 + d)/2, (\lambda_0 + \Delta_1^2)/2) \md s \\
    &= \int \NPDF(x_2 ; f_{\theta}(x_1), u \cdot \omega_1 Q_0) \IGPDF(u ; (\lambda_0 + d)/2, (\lambda_0 + d)/2) \md u
\end{align}
where we find $p(x_2 \given x_1, u) = \NPDF(x_2 ; f_{\theta}(x_1), u \cdot \omega_1 Q_0)$. We then have that
\begin{align}
    \tilde{p}(x_2 \given y_1, u) &= \int \tilde{p}(x_2 \given x_1, u) \tilde{p}(x_1 \given y_1, u) \md x_1 \\
    &= \int \NPDF(x_2 ; f_{\theta}(x_1), u\cdot \omega_1 Q_0) \NPDF(x_1 ; \mu_1, u P_1) \md x_1,
\end{align}
which we recognize to be analogous to \eqref{eq:coefpredint1}. This expression also reveals how the noise covariance $Q_0$ is updated, as we may simply define $Q_1 = \omega_1 Q_0$. This gives $\tilde{p}(x_2 \given y_1, u) = \NPDF(x_2 ; \bar{\mu}_2, u \bar{P}_2)$ with $\bar{\mu}_2$ and $\bar{P}_2$ analogous to $\bar{\mu}_1$ and $\bar{P}_1$ above.

Similar reasoning can be applied to obtain the predictive distribution of $y_2$. From the dictionary update we have obtained $\tilde{p}(c \given y_1) = \TPDF(c ; c_1, L_1, \lambda_0 + d)$, which we can also write as a scale mixture with $u$ as $\tilde{p}(c \given y_1, u) = \NPDF(c ; c_1, u L_1)$. The model definition gives $p(y_2 \given x_2, C, s) = \NPDF(y_2 ; C x_2, s R_0)$. Again writing this in terms of $u$ by using the change of variables $u = \omega_1^{-1} s$ and Lemma~\ref{lem:igtransform} and \eqref{eq:rob_s_given_y1}, yields $p(y_2 \given x_2, C, u) = \NPDF(y_2 ; C x_2, u \cdot \omega_1 R_0)$. Combining these expressions gives
\begin{align}
    \tilde{p}(y_2 \given c, y_1, u) = \int \tilde{p}(y_2 \given x_2, C, u) \tilde{p}(c \given y_1, u) \md c = \int \NPDF(y_2 ; (x_2^{\top} \otimes I_d) c, u \cdot \omega_1 R_0) \NPDF(c ; c_1, u L_1) \md c,
\end{align}
which is analogous to \eqref{eq:y1_given_x1_s}. We also see that we can define $R_1 = \omega_1 R_0$ to update the measurement noise covariance.

We observe in the above derivation that after completing an entire iteration we have obtained a new scale variable $u \sim \IGPDF(u ; (\lambda_0 + d)/2, (\lambda_0 + d)/2)$, and that we have found update rules for the noise covariances $Q$ and $R$. This procedure is repeated at every step, and we can define appropriate notation for this process by setting $s_0 = s$ and $s_1 = u$, and generally have scale variables $s_k \sim \IGPDF(s_k ; \lambda_k/2, \lambda_k/2)$ with $\lambda_k = \lambda_{k-1} + d$. Thus, $s_k = \omega_k^{-1} s_{k-1}$ with $\omega_k = (\lambda_{k-1} + \Delta_k^2)/(\lambda_{k-1} + d)$, which corresponds to \cite{tronarp2019student}. The noise covariances are clearly updated as $Q_k = \omega_k Q_{k-1}$ and $R_k = \omega_k R_{k-1}$. Algorithm~\ref{alg:rPSMF} summarizes the steps of robust PSMF, including steps for parameter estimation using both the iterative and recursive approaches.

\begin{algorithm}[t]
	\small
\begin{algorithmic}[1]
\caption{Iterative and recursive rPSMF}\label{alg:rPSMF}
\State Initialize $\gamma$, $\theta_0$, $C_0$, $V_0$, $\mu_0$, $P_0$, 
$Q_{0}$, $R_{0}$.
\For{$i \geq 1$}
\State $C_0 = C_T$, $\mu_0 = \mu_T$, $P_0 = P_T$, $V_0 = V_T$.
\For{$1 \leq k \leq T$}
\State Predictive mean of $x_k$: $\bar{\mu}_k = f_{\theta_{i-1}}(\mu_{k-1})$ or $\bar{\mu}_k = f_{\theta_{k-1}}(\mu_{k-1})$
\State Predictive covariance of $x_k$ 
\begin{equation*}
	\bar{P}_k = F_k P_{k-1} F_k^\top + Q_{k-1} , \quad \mathrm{ where } \quad 
	F_k = \frac{\partial f(x)}{\partial x} \Big|_{x = \bar{\mu}_{k-1}}
\end{equation*}
\State Compute scaling factor for the dictionary update
\begin{align*}
    \varphi_k = \frac{\lambda_{k-1}}{\lambda_{k-1} + d} + \frac{(y_k - C_{k-1}\bar{\mu}_k)^{\top}(y_k - C_{k-1}\bar{\mu}_k)}{(\lambda_{k-1} + d)(\bar{\mu}_k^{\top} V_{k-1} \bar{\mu}_k + \eta_k)}
\end{align*}
\hskip\algorithmicindent\hskip\algorithmicindent where $\eta_k = \Tr(C_{k-1} \bar{P}_k C_{k-1}^{\top} + R_{k-1})/d$.
\State Mean and covariance updates of the dictionary
\begin{align*}
C_k = C_{k-1} + \frac{(y_k - C_{k-1} \bar{\mu}_k) \bar{\mu}_k^\top V_{k-1}^{\top}}{\bar{\mu}_k^\top V_{k-1} \bar{\mu}_k + \eta_k} \quad \text{and} \quad V_k = \varphi_k \left(
	V_{k-1} - \frac{
		V_{k-1} \bar{\mu}_k \bar{\mu}_k^\top V_{k-1}
	}{
		\bar{\mu}_k^\top V_{k-1} \bar{\mu}_k + \eta_k
	}  \right)
\end{align*}
\State Compute scaling factor for the coefficient update
\begin{align*}
    \omega_k = \frac{\lambda_{k-1} + (y_k - C_{k-1}\bar{\mu}_k)^{\top} S_k^{-1} (y_k - C_{k-1}\bar{\mu}_k)}{\lambda_{k-1} + d}
\end{align*}
\hskip\algorithmicindent\hskip\algorithmicindent where $S_k = C_{k-1} 
\bar{P}_k C_{k-1}^\top + \bar{R}_{k-1}$ and $\bar{R}_{k-1} = R_{k-1} + \bar{\mu}_k^{\top}V_{k-1}\bar{\mu}_k \otimes I_d$.
\State Mean and covariance updates of coefficients
\begin{align*}
\mu_k = \bar{\mu}_k + \bar{P}_k C_{k-1}^\top S_k^{-1} (y_k - C_{k-1} 
\bar{\mu}_k) \quad \text{and} \quad P_k = \omega_k (\bar{P}_k - \bar{P}_k C_{k-1}^\top 
S_k^{-1} C_{k-1} \bar{P}_k)
\end{align*}
\State Update noise covariances: $Q_k = \omega_k Q_{k-1}$ and $R_k = \omega_k R_{k-1}$
\State Update degrees of freedom: $\lambda_k = \lambda_{k-1} + d$.
\State Parameter update: $\theta_k = \theta_{k-1} + \gamma \nabla \log \tilde{p}_{\theta}(y_k | y_{1:k-1}) \big|_{\theta = \theta_{k-1}}$ \Comment{\textcolor{gray}{recursive version}}
\EndFor
\State Parameter update: $\theta_i = \theta_{i-1} + \gamma \sum_{k=1}^T \nabla \log \tilde{p}_{\theta}(y_k | y_{1:k-1}) \big|_{\theta = \theta_{i-1}}.$ \Comment{\textcolor{gray}{iterative version}}
\EndFor
\end{algorithmic}
\end{algorithm}

For completeness, we give the approximate negative marginal likelihood $\tilde{p}_{\theta}(y_k \given y_{1:k-1})$, similar to Sec.~\ref{sec:app:Gradients}. It follows that
\begin{align}
    \tilde{p}_{\theta}(y_k \given y_{1:k-1}) &= \iint \tilde{p}(y_k \given y_{1:k-1}, c, s_{k-1}) \tilde{p}(c \given y_{1:k-1}, s_{k-1}) \md c \md s_{k-1} \\
    &= \iint \NPDF(y_k ; H_k c, s_{k-1} G_k) \NPDF(c ; c_{k-1}, s_{k-1} L_{k-1}) \md c \md s_{k-1} \\
    &= \int \NPDF(y_k ; H_k c_{k-1}, s_{k-1} (H_k L_{k-1} H_k^{\top} + G_k)) \md s_{k-1} \\
    &= \TPDF(y_k ; H_k c_{k-1}, H_k L_{k-1} H_k^{\top} + G_k, \lambda_{k-1})
\end{align}
With $H_k c_{k-1} = C_{k-1} \bar{\mu}_k$ and $H_k L_{k-1} H_k^{\top} + G_k = (\bar{\mu}_k^{\top} V_{k-1} \bar{\mu}_k + \eta_k) \otimes I_d$ where $\bar{\mu}_k = f_{\theta}(\mu_{k-1})$, we find after a brief algebraic exercise that
\begin{align}
    - \log \tilde{p}_{\theta}(y_k \given y_{1:k-1}) &\stackrel{c}{=} \frac{d}{2} \log \left( \| f_{\theta}(\mu_{k-1}) \|^2_{V_{k-1}} + \eta_k \right) \\
    &\qquad\qquad + \left( \frac{\lambda_{k-1} + d}{2} \right) \log \left( 1 + \frac{ \| y_k - C_{k-1} f_{\theta}(\mu_{k-1}) \|^2 }{  \lambda_{k-1} \left( \| f_{\theta}(\mu_{k-1}) \|^2_{V_{k-1}}  + \eta_k \right)  }  \right)
\end{align}
where $\stackrel{c}{=}$ again denotes equality up to terms independent of $\theta$. Finally, we note that handling missing values in rPSMF is straightforward and follows the same reasoning as for PSMF in Supp.~\ref{app:Missing}.

\section{ADDITIONAL DETAILS FOR THE EXPERIMENTS}

\subsection{Experiment 1}\label{app:Experiment1}

\begin{figure}[tb]
\centering
\begin{subfigure}[b]{.45\textwidth}
\includegraphics[width=\textwidth,height=5cm]{images/paper_plot_rpsmf_fit.pdf}
\caption{Observed time series (\textcolor{SynBlue}{blue}) with unobserved future data (\textcolor{SynYellow}{yellow}) and the reconstruction (\textcolor{SynRed}{red}).}
\label{fig:app_figDatavsPredSynt}
\end{subfigure}\qquad
\begin{subfigure}[b]{.45\textwidth}
\includegraphics[width=\textwidth,height=2cm]{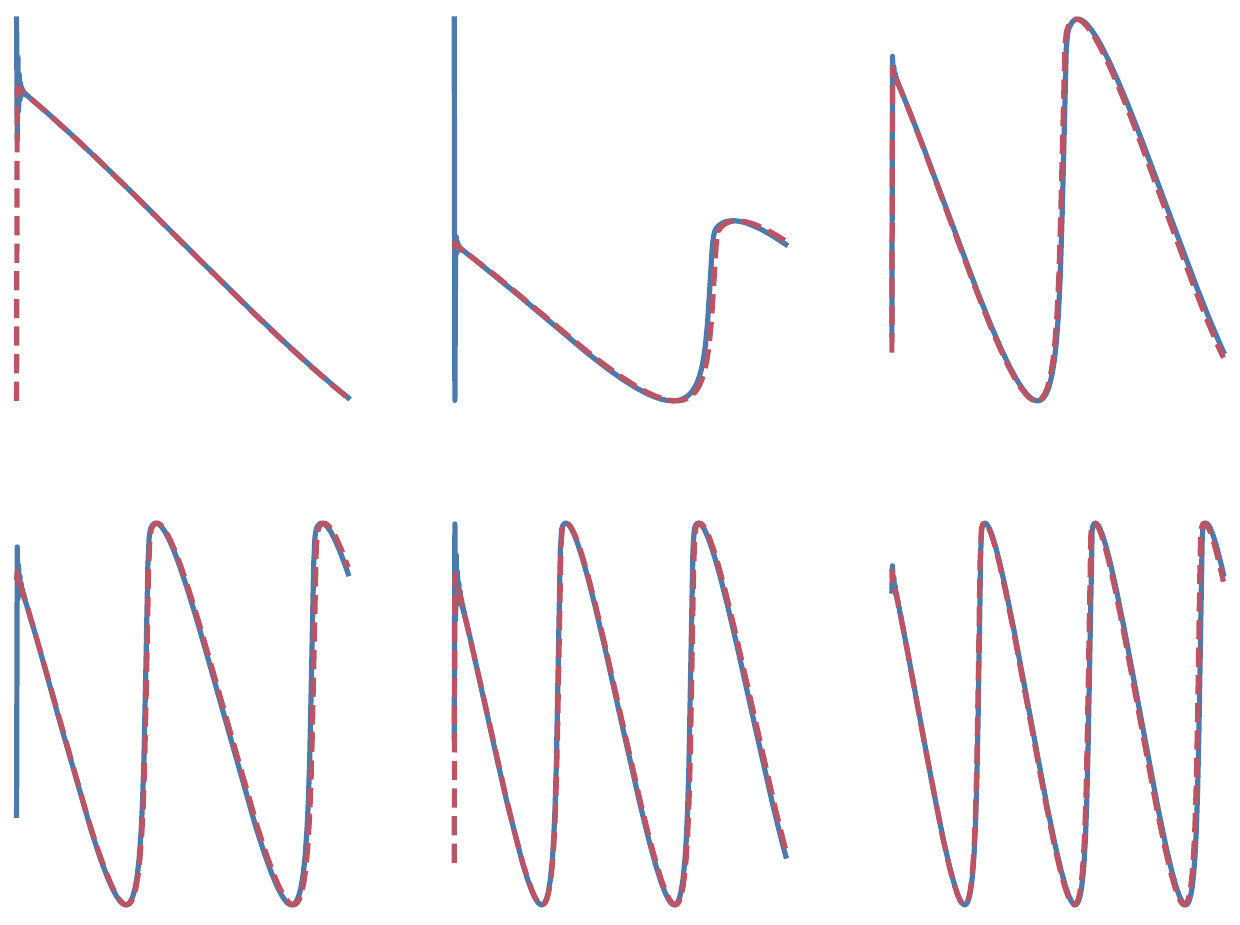}
\caption{True (\textcolor{SynBlue}{blue}) and predicted (\textcolor{SynRed}{red}) subspace.\label{fig:app_figTrueSubspaceSynt}}
\vspace{2ex}
\includegraphics[width=\textwidth,height=2cm]{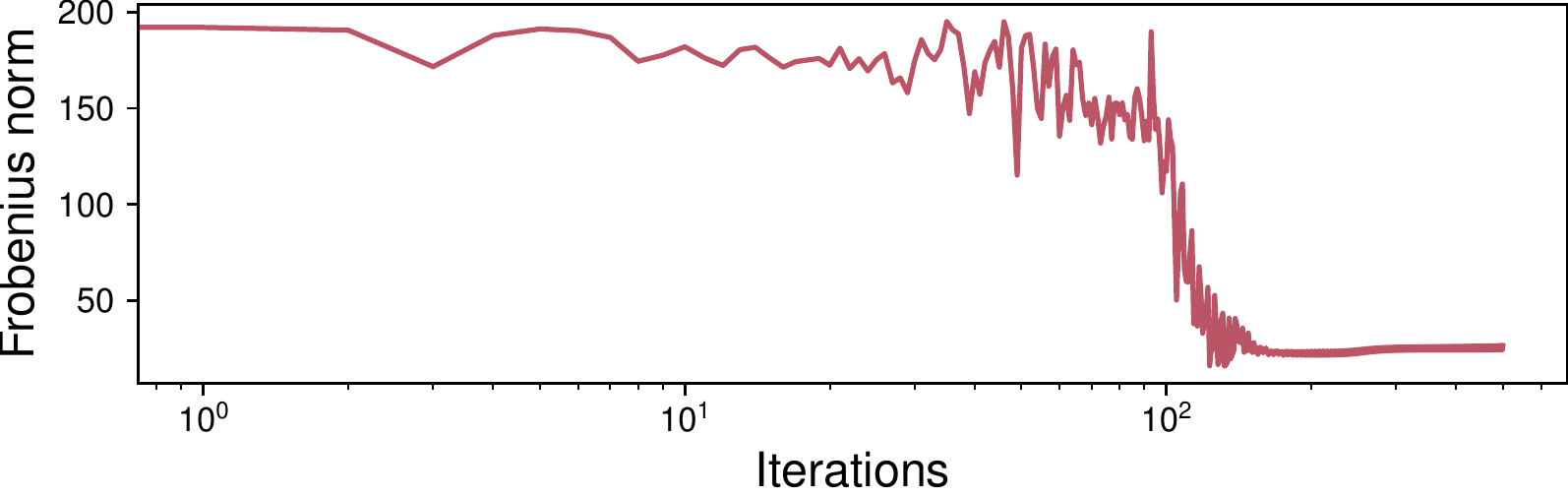}
\caption{Reconstruction error $\|Y - CX\|_F^2$. \label{fig:app_figFrobeniusSynt}}
\end{subfigure}
\caption{Fitting rPSMF on synthetic data with $t$-distributed noise. Figure (a) illustrates the fit to the observed and unobserved measurements. Figure (b) contains the true and reconstructed subspace, and (c) shows the reconstruction error over outer iterations of the iterative algorithm.\label{fig:app_synthetic_rPSMF}}
\end{figure}

\textbf{Optimization} In this experiment, we have used the Adam optimizer \cite{kingma2015adam}. In particular, instead of implementing the gradient step \eqref{eq:offlineSGD}, we replace it with the Adam optimizer. In order to do so, we define the gradient as $g_i = \nabla \log \tilde{p}_\theta(y_{1:n}) \Big|_{\theta = \theta_{i-1}}$. Upon computing the gradient $g_i$, we first compute the running averages
\begin{align}
m_i &= \beta_1 m_{i-1} + (1-\beta_1) g_i \\
v_i &= \beta_2  v_{i-1} + (1 - \beta_2) (g_i \odot g_i),
\end{align}
which is then corrected as
\begin{align}
\hat{m}_i &= \frac{m_i}{1 - \beta_1^i} \\
\hat{v}_i &= \frac{v_i}{1-\beta_2^i}.
\end{align}
Finally the parameter update is computed as
\begin{align}
\theta_i = \mathsf{Proj}_\Theta\left( \theta_{i-1} + \gamma \frac{\hat{m}_i}{\sqrt{\hat{v}_i} + \epsilon}\right),
\end{align}
where $\mathsf{Proj}$ denotes the projection operator which constrains the parameter to stay positive in each dimension where $\Theta = \bR_+ \times \cdots \times \bR_+ \subset \bR^6$ which is implemented by simple $\max$ operators. We choose the standard parameterization with $\gamma = 10^{-3}, \beta_1 = 0.9, \beta_2 = 0.999$ and $\epsilon = 10^{-8}$.

In these experiments we use an observed time series of length 500 and a series of unobserved future data of length 250. Fig.~\ref{fig:app_synthetic_rPSMF} corresponds to the figure in the main text, but additionally shows how the underlying subspace is recovered and how the Frobenius norm between the reconstructed data and the true data decreases with the number of iterations. Fig.~\ref{fig:app_synthetic_PSMF} shows a similar result for the PSMF method on normally-distributed data.

\begin{figure}[tb]
\centering
\begin{subfigure}[b]{.45\textwidth}
\includegraphics[width=\textwidth,height=5cm]{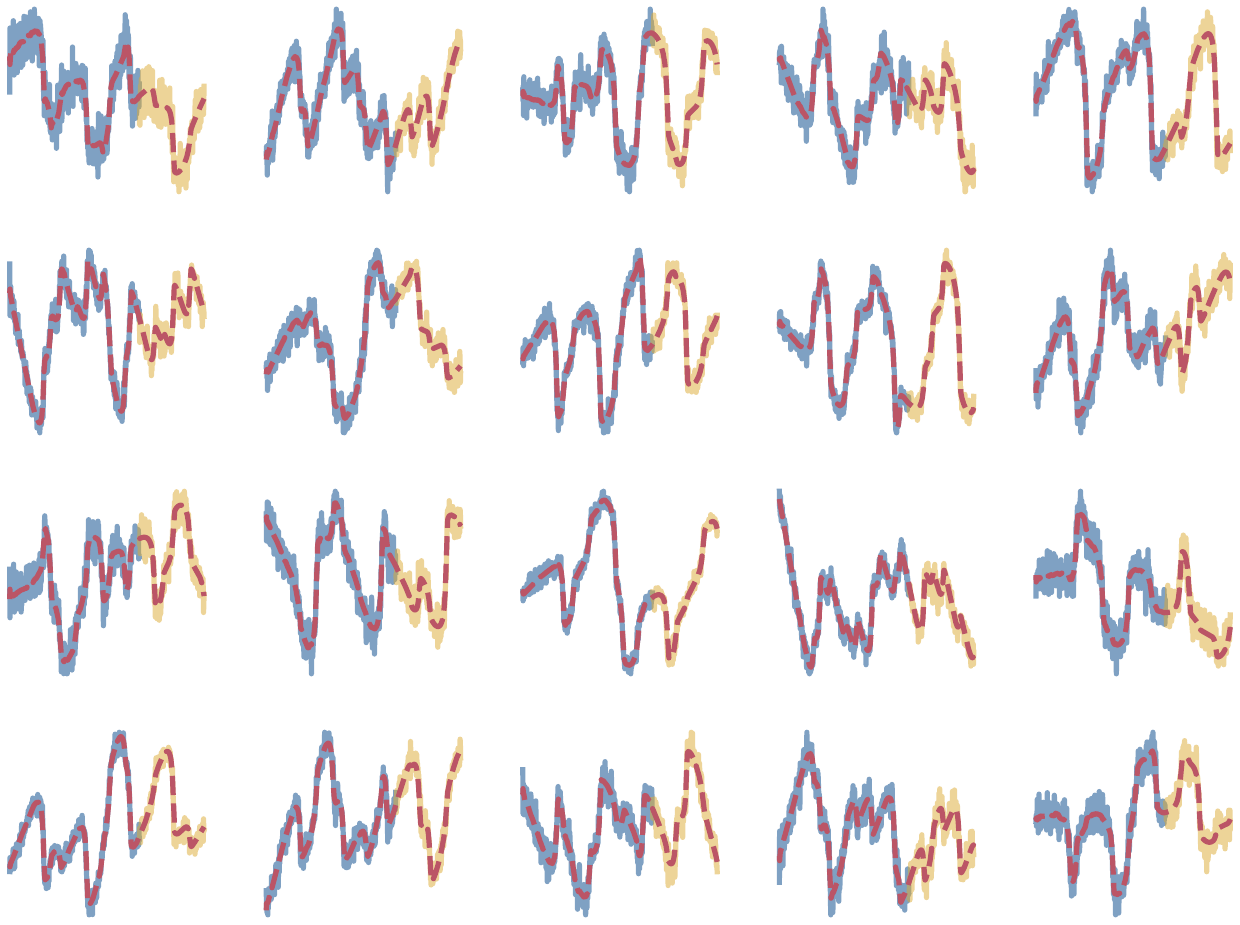}
\caption{Observed time series (\textcolor{SynBlue}{blue}) with unobserved future data (\textcolor{SynYellow}{yellow}) and the reconstruction (\textcolor{SynRed}{red}).}
\end{subfigure}\qquad
\begin{subfigure}[b]{.45\textwidth}
\includegraphics[width=\textwidth,height=2cm]{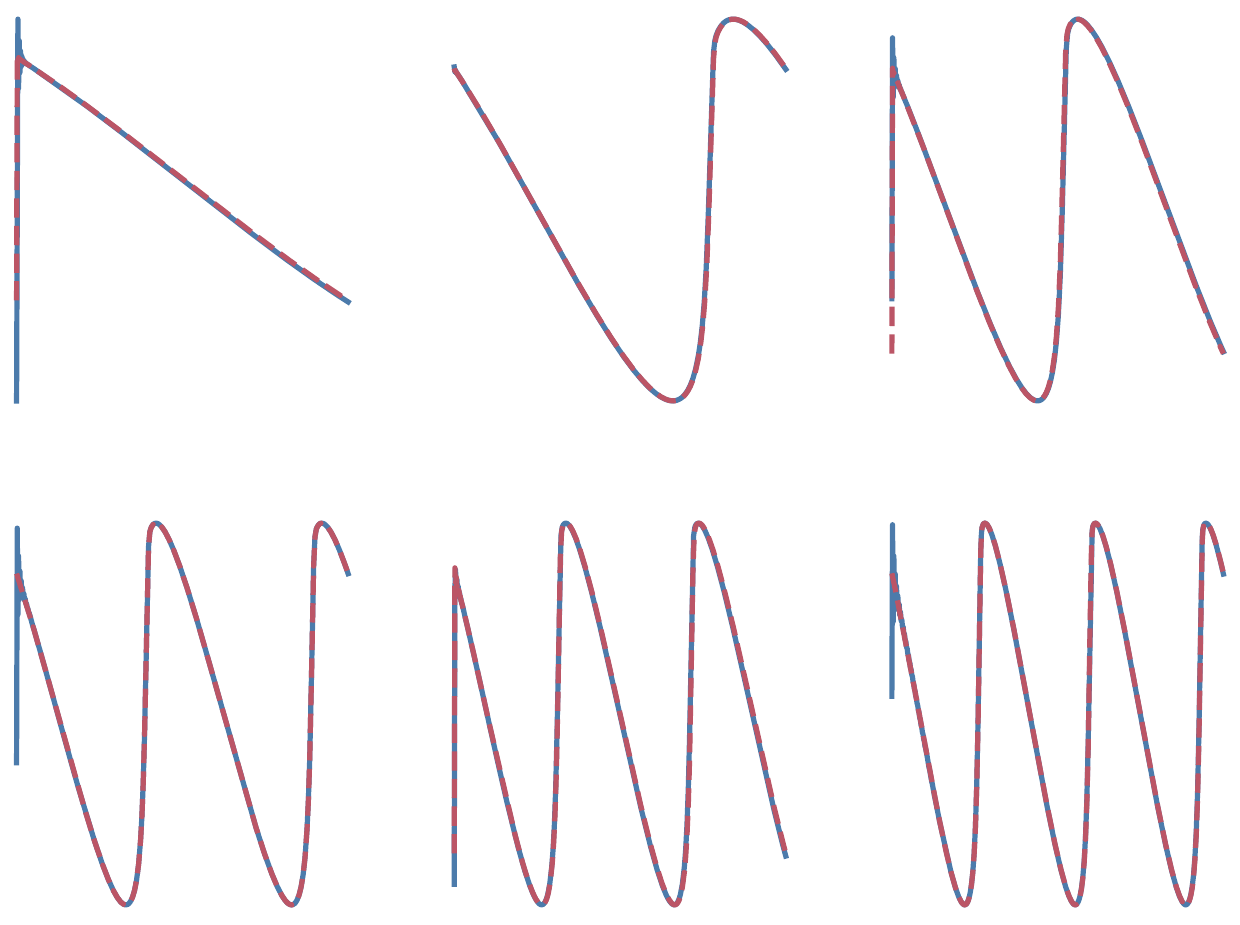}
\caption{True (\textcolor{SynBlue}{blue}) and predicted (\textcolor{SynRed}{red}) subspace.}
\vspace{2ex}
\includegraphics[width=\textwidth,height=2cm]{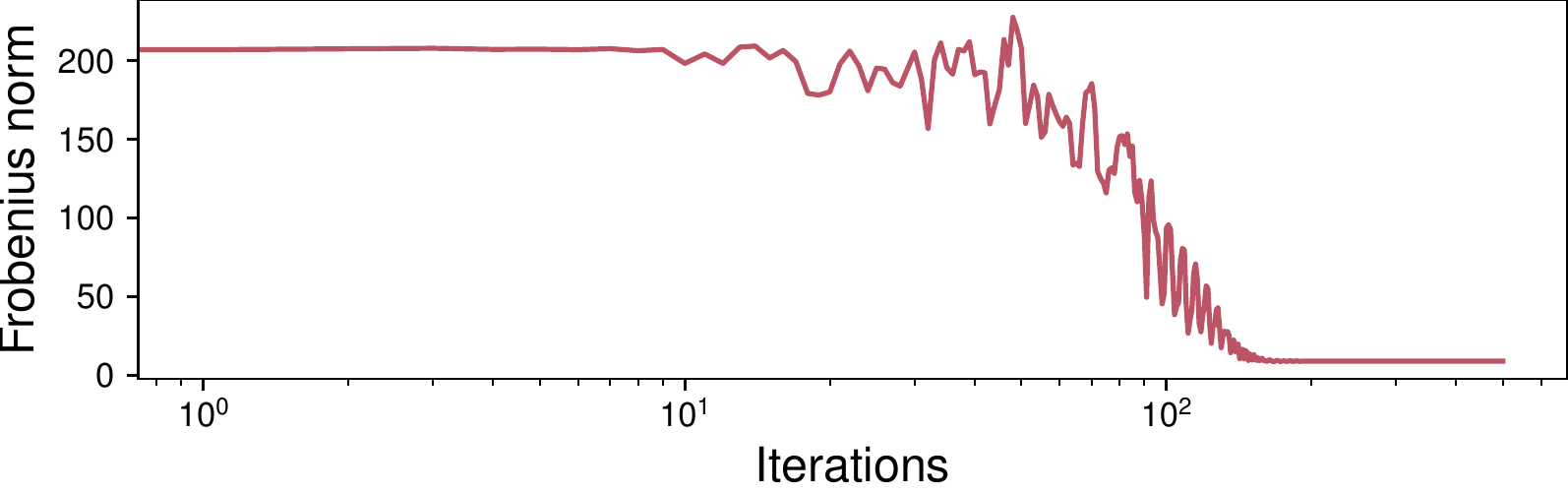}
\caption{Reconstruction error $\|Y - CX\|_F^2$.}
\end{subfigure}
\caption{Fitting PSMF on synthetic data with normally distributed noise. Figure (a) illustrates the fit to the observed and unobserved measurements. Figure (b) contains the true and reconstructed subspace, and (c) shows the reconstruction error over outer iterations of the iterative algorithm.\label{fig:app_synthetic_PSMF}}
\end{figure}

\subsection{Experiment 2}\label{app:Experiment2}

\subsubsection{Data generation and the experimental setup}
We generate periodic time series using pendulum differential equations as the true subspace. For this experiment, we generate $d = 20$ dimensional data where $d_2 = 3$ of them undergo a structural change. In order to test the method, we generate $1000$ synthetic datasets. One such dataset is given in Fig.~\ref{fig:syntheticdata}.
\begin{figure}[h]
    \centering
    \includegraphics[scale=0.74]{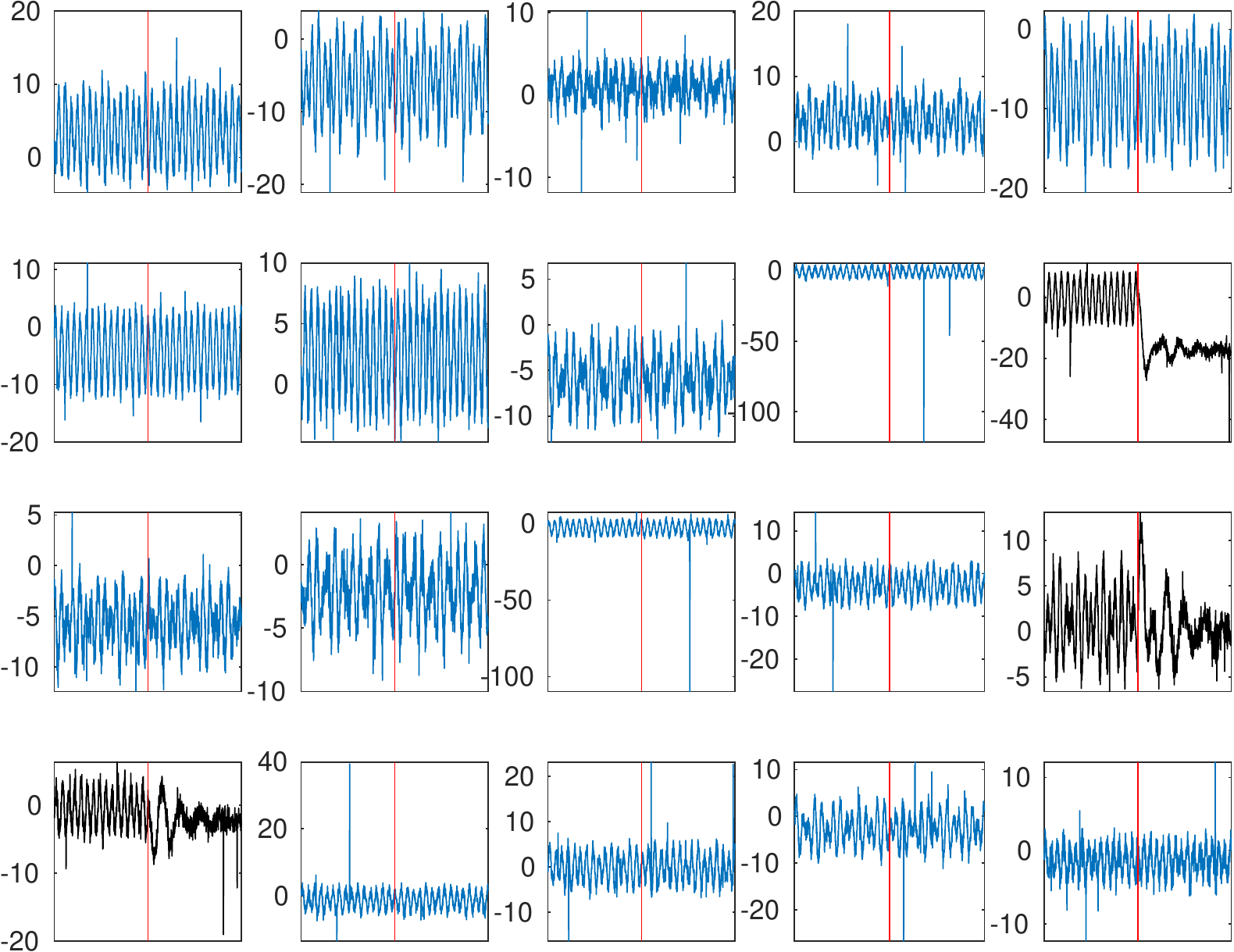}
    \caption{One instance of the $1000$ different synthetic datasets used in Sec.~\ref{sec:ExperimentChangepoint}. The dimensions which exhibit a structural change can be seen in black. The data contain outliers and the true changepoint can be seen as marked by the vertical red line.}
    \label{fig:syntheticdata}
\end{figure}
We generate data with $n = 1200$ and use the data after the data point $n_0 = 400$ to estimate changepoints, as PSMF has to converge to a stable regime before it can be used to detect changepoints. The true changepoint is at $n_c = 601$.
\subsubsection{The GP subspace model}
In this subsection, we provide the details of the discretization of the Mat\'{e}rn-$3/2$ SDE. Particularly, we consider the SDE \cite{sarkka2013spatiotemporal}
\begin{align}\label{eq:SDE_GP_app}
\frac{\md \mathsf{x}_i(t)}{\md t} = \mathsf{F}\mathsf{x}_i(t) + \begin{bmatrix}
0 \\
1
\end{bmatrix} w_i(t)
\end{align}
where $\mathsf{x}_i(t) = [x_i(t), \md x_i(t) / \md t]$ and $\kappa = \sqrt{2\nu}/{\ell}$ and
\begin{align}
\mathsf{F} = \begin{bmatrix}
0 & 1 \\
-\kappa^2 & - 2\kappa
\end{bmatrix}.
\end{align}
Given a step-size $\gamma$, the SDE~\eqref{eq:SDE_GP_app} can be written as a linear dynamical system
\begin{align}
x_{i,k} = A_i x_{i,k-1} + Q_i^{1/2} u_{i,k}
\end{align}
where $A_i = \textnormal{expm}(\gamma F)$ where $\textnormal{expm}$ denotes the matrix exponential and $Q_i = P_\infty - A_i P_\infty A_i^\top$ and
\begin{align}
P_\infty =\begin{bmatrix}
\sigma^2 & 0 \\
0 & 3 \sigma^2 / \ell^2
\end{bmatrix}.
\end{align}
Finally, we construct our dynamical system as
\begin{align}
x_k = A x_{k-1} + Q^{1/2} u_k
\end{align}
where $x_k = [x_{1,k},\ldots,x_{r,k}]^\top \in \bR^{2r}$ and
\begin{align}
A = I_r \otimes A_i \quad \quad \textnormal{and} \quad \quad Q = I_r \otimes Q_i.
\end{align}
Using these system matrices, we define $H_i = [1,0]$ and $H = I_r \otimes H_i$ and finally define the probabilistic model
\begin{align}
p(C) &= \mathcal{MN}(C;C_0,I_d,V_0), \\
p(x_0) &= \NPDF(x_0; \mu_0, P_0), \\
p(x_k | x_{k-1}) &= \NPDF(x_k;Ax_{k-1}, Q),  \\
p(y_k | x_k, C) &= \NPDF(y_k; C H x_k, R).
\end{align}
Inference in this model can be done via a simple modification of the Algorithm~\ref{algDF} where $H$ matrix is involved in the computations. Fig.~\ref{fig:GPfeatures} illustrates the learned GP features $r = 4$ and two change points.

\begin{figure}[tb]
\centering
\includegraphics[height=4cm]{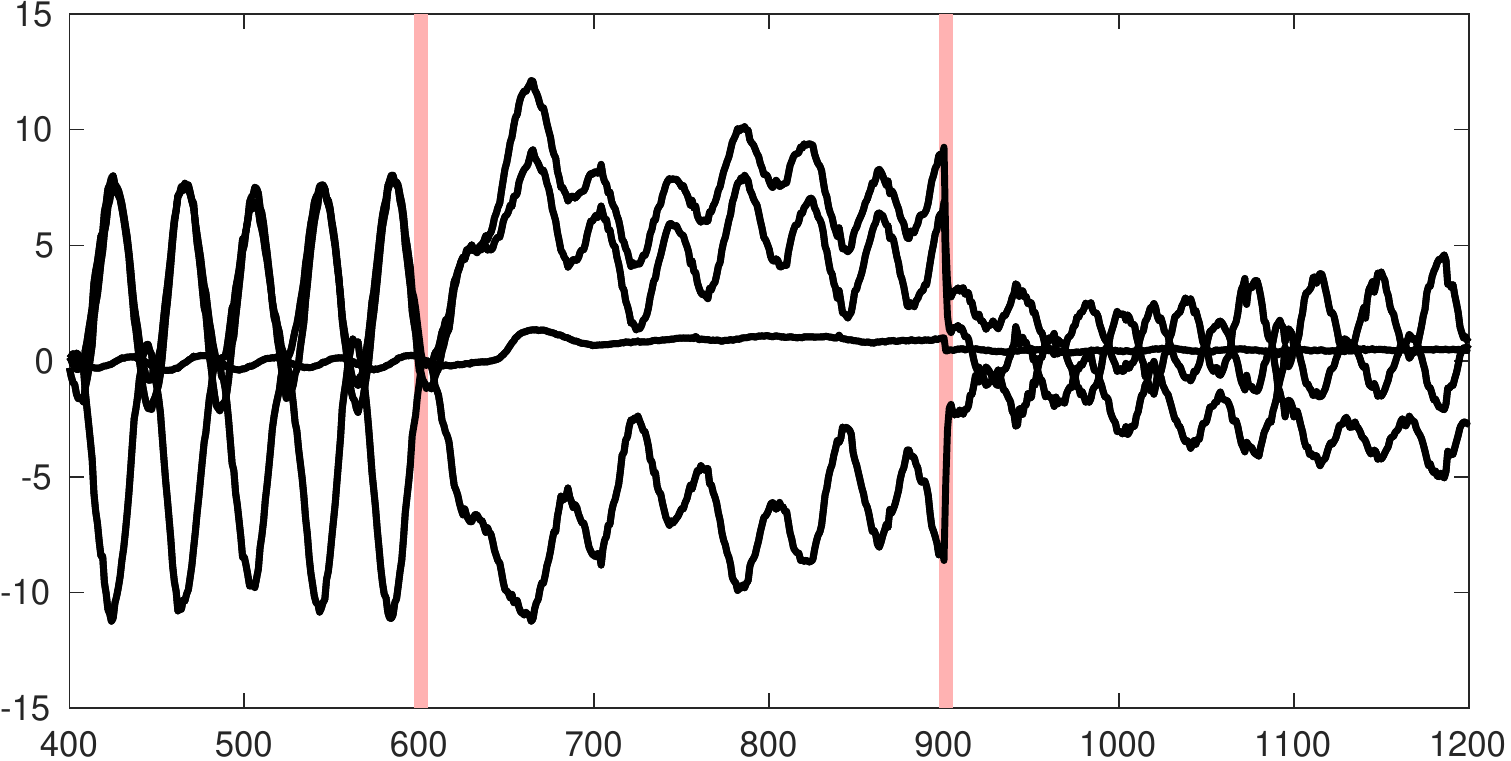}
\caption{An illustration of the learned GP features vs. true changepoints for $r=4$ and two changepoints.}
\label{fig:GPfeatures}
\end{figure}

\subsection{Experiment 3}
\label{app:Experiment3}

All experiments where run on a Linux machine with an AMD Ryzen 5 3600 processor and 32GB of memory. Additional results on different missing percentages are shown in Table~\ref{tab:results_real_impute_app} and Table~\ref{tab:results_real_cover_app}. We again observe excellent imputation performance of the proposed methods.

\begin{table}[ht]
\caption{Imputation error and runtime on several datasets using 20\% and 40\% missing values, averaged over 100 random repetitions. An asterisk marks offline methods.\label{tab:results_real_impute_app}}
	\small
    \centering
    \begin{subtable}[t]{.9\textwidth}
        \caption{20\% missing data\label{tab:results_real_impute_20}}
\begin{tabular}{lrrrrrrrrrr}
\toprule
 & \multicolumn{5}{c}{Imputation RMSE} & \multicolumn{5}{c}{Runtime (s)} \\\cmidrule(lr){2-6} \cmidrule(lr){7-11}
 & NO$_2$ & PM10 & PM25 & S\&P500 & Gas & NO$_2$ & PM10 & PM25 & S\&P500 & Gas \\ \cmidrule(lr){2-6} \cmidrule(lr){7-11}
PSMF & $\underset{{\scriptscriptstyle (0.10)}}{\textbf{5.52}}$ & $\underset{{\scriptscriptstyle (0.38)}}{\textbf{7.26}}$ & $\underset{{\scriptscriptstyle (0.52)}}{3.42}$ & $\underset{{\scriptscriptstyle (1.93)}}{9.95}$ & $\underset{{\scriptscriptstyle (0.57)}}{\textbf{4.19}}$ & 2.71 & 2.59 & 1.92 & 9.42 & 101.16\\
rPSMF & $\underset{{\scriptscriptstyle (0.14)}}{5.53}$ & $\underset{{\scriptscriptstyle (0.46)}}{7.47}$ & $\underset{{\scriptscriptstyle (0.52)}}{\textbf{3.40}}$ & $\underset{{\scriptscriptstyle (1.41)}}{\textbf{9.29}}$ & $\underset{{\scriptscriptstyle (0.51)}}{4.56}$ & 2.91 & 2.73 & 2.03 & 13.98 & 122.57\\
MLE-SMF & $\underset{{\scriptscriptstyle \;\;(0.51)}}{11.03}$ & $\underset{{\scriptscriptstyle (0.37)}}{9.46}$ & $\underset{{\scriptscriptstyle (0.63)}}{4.81}$ & $\underset{{\scriptscriptstyle \;\;(1.02)}}{30.23}$ & $\underset{{\scriptscriptstyle \;\;(14.84)}}{87.12}$ & 2.48 & 2.39 & 1.71 & 9.52 & 92.09\\
TMF & $\underset{{\scriptscriptstyle (0.14)}}{7.60}$ & $\underset{{\scriptscriptstyle (0.30)}}{7.95}$ & $\underset{{\scriptscriptstyle (0.50)}}{4.43}$ & $\underset{{\scriptscriptstyle \;\;(1.00)}}{34.96}$ & $\underset{{\scriptscriptstyle \;\;(8.85)}}{73.70}$ & 1.03 & 0.97 & 0.72 & 4.19 & 35.35\\
PMF* & $\underset{{\scriptscriptstyle \;\;(0.08)}}{10.47}$ & $\underset{{\scriptscriptstyle \;\;(0.26)}}{10.46}$ & $\underset{{\scriptscriptstyle (0.48)}}{3.97}$ & $\underset{{\scriptscriptstyle \;\;(1.80)}}{40.07}$ & $\underset{{\scriptscriptstyle \;\;(0.06)}}{23.54}$ & 2.14 & 1.90 & 0.68 & 3.12 & 31.78\\
BPMF* & $\underset{{\scriptscriptstyle (0.18)}}{9.03}$ & $\underset{{\scriptscriptstyle (0.28)}}{8.39}$ & $\underset{{\scriptscriptstyle (0.49)}}{3.61}$ & $\underset{{\scriptscriptstyle \;\;(0.93)}}{27.36}$ & $\underset{{\scriptscriptstyle \;\;(0.17)}}{17.70}$ & 3.11 & 4.48 & 3.05 & 4.15 & 92.50\\
\bottomrule
\end{tabular}
    \end{subtable}
    \vskip\baselineskip
    \begin{subtable}[t]{.9\textwidth}
        \caption{40\% missing data\label{tab:results_real_impute_40}}
\begin{tabular}{lrrrrrrrrrr}
\toprule
 & \multicolumn{5}{c}{Imputation RMSE} & \multicolumn{5}{c}{Runtime (s)} \\\cmidrule(lr){2-6} \cmidrule(lr){7-11}
 & NO$_2$ & PM10 & PM25 & S\&P500 & Gas & NO$_2$ & PM10 & PM25 & S\&P500 & Gas \\ \cmidrule(lr){2-6} \cmidrule(lr){7-11}
PSMF & $\underset{{\scriptscriptstyle (0.18)}}{6.06}$ & $\underset{{\scriptscriptstyle (0.28)}}{7.72}$ & $\underset{{\scriptscriptstyle (0.23)}}{3.77}$ & $\underset{{\scriptscriptstyle \;\;(3.06)}}{13.87}$ & $\underset{{\scriptscriptstyle (1.63)}}{\textbf{8.75}}$ & 2.77 & 2.62 & 1.92 & 9.12 & 100.68\\
rPSMF & $\underset{{\scriptscriptstyle (0.27)}}{\textbf{5.96}}$ & $\underset{{\scriptscriptstyle (0.57)}}{\textbf{7.68}}$ & $\underset{{\scriptscriptstyle (0.29)}}{\textbf{3.67}}$ & $\underset{{\scriptscriptstyle \;\;(4.39)}}{\textbf{12.36}}$ & $\underset{{\scriptscriptstyle (2.28)}}{9.03}$ & 2.92 & 2.77 & 2.02 & 13.30 & 109.38\\
MLE-SMF & $\underset{{\scriptscriptstyle \;\;(0.49)}}{11.30}$ & $\underset{{\scriptscriptstyle (0.30)}}{9.55}$ & $\underset{{\scriptscriptstyle (0.31)}}{4.93}$ & $\underset{{\scriptscriptstyle \;\;(0.80)}}{30.14}$ & $\underset{{\scriptscriptstyle \;\;\;(26.65)}}{125.54}$ & 2.54 & 2.38 & 1.70 & 9.59 & 85.11\\
TMF & $\underset{{\scriptscriptstyle (0.12)}}{7.90}$ & $\underset{{\scriptscriptstyle (0.21)}}{8.27}$ & $\underset{{\scriptscriptstyle (0.31)}}{4.86}$ & $\underset{{\scriptscriptstyle \;\;(0.76)}}{34.78}$ & $\underset{{\scriptscriptstyle \;\;(10.60)}}{66.27}$ & 0.98 & 0.97 & 0.73 & 4.13 & 32.01\\
PMF* & $\underset{{\scriptscriptstyle \;\;(0.05)}}{10.54}$ & $\underset{{\scriptscriptstyle \;\;(0.15)}}{10.53}$ & $\underset{{\scriptscriptstyle (0.13)}}{4.11}$ & $\underset{{\scriptscriptstyle \;\;(1.81)}}{41.53}$ & $\underset{{\scriptscriptstyle \;\;(0.06)}}{24.12}$ & 1.73 & 1.51 & 0.54 & 2.43 & 24.75\\
BPMF* & $\underset{{\scriptscriptstyle (0.21)}}{9.46}$ & $\underset{{\scriptscriptstyle (0.18)}}{8.64}$ & $\underset{{\scriptscriptstyle (0.12)}}{3.72}$ & $\underset{{\scriptscriptstyle \;\;(0.64)}}{27.91}$ & $\underset{{\scriptscriptstyle \;\;(0.37)}}{19.10}$ & 4.26 & 4.07 & 2.92 & 3.16 & 82.44\\
\bottomrule
\end{tabular}
    \end{subtable}
\end{table}

\begin{table}[ht]
\caption{Average coverage proportion of the missing data by the $2\sigma$ uncertainty bars of the posterior predictive estimates for 20\% and 40\% missing values, averaged over 100 repetitions. \label{tab:results_real_cover_app}}
    \small
    \centering
    \begin{subtable}[t]{.45\textwidth}
        \caption{20\% missing data\label{tab:results_real_cover_20}}
\begin{tabular}{lccccc}
\toprule
 & NO$_2$ & PM10 & PM25 & S\&P500 & Gas \\
\cmidrule(lr){2-6}
PSMF & 0.79 & 0.79 & \textbf{0.93} & 0.85 & \textbf{0.93} \\
rPSMF & \textbf{0.89} & \textbf{0.92} & 0.91 & \textbf{0.86} & 0.90 \\
MLE-SMF & 0.46 & 0.59 & 0.83 & 0.51 & 0.61 \\
\bottomrule
\end{tabular}
    \end{subtable}
    \qquad
    \begin{subtable}[t]{.45\textwidth}
        \caption{40\% missing data\label{tab:results_real_cover_40}}
\begin{tabular}{lccccc}
\toprule
 & NO$_2$ & PM10 & PM25 & S\&P500 & Gas \\
\cmidrule(lr){2-6}
PSMF & 0.71 & 0.73 & \textbf{0.89} & 0.78 & \textbf{0.83} \\
rPSMF & \textbf{0.79} & \textbf{0.84} & 0.81 & \textbf{0.79} & 0.79 \\
MLE-SMF & 0.40 & 0.53 & 0.77 & 0.44 & 0.49 \\
\bottomrule
\end{tabular}
    \end{subtable}
\end{table}

\section{CONVERGENCE DISCUSSION}\label{app:Convergence}
\begin{figure}[!tb]
	\begin{center}
		\includegraphics[scale=0.4]{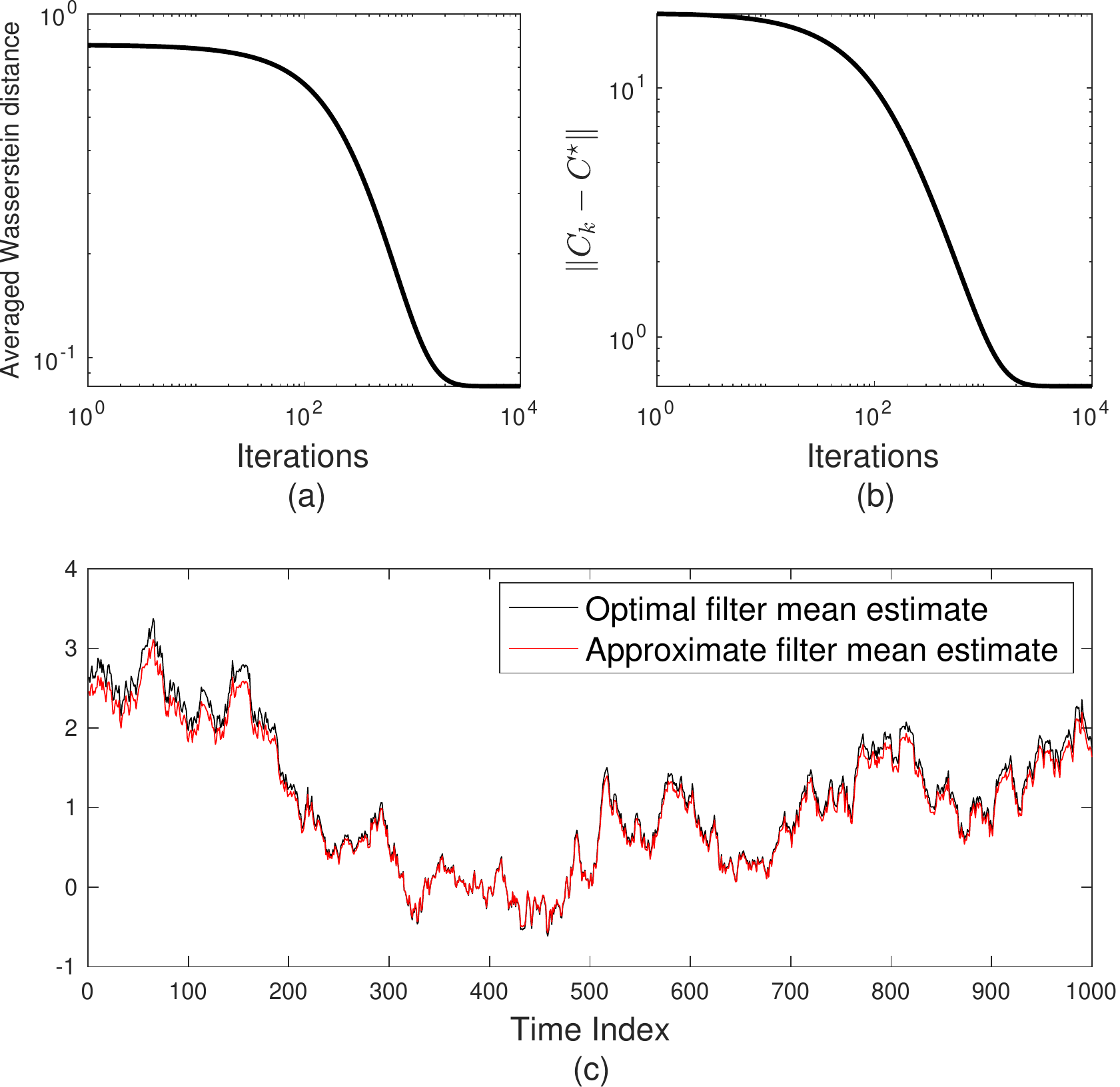}
	\end{center}
	\caption{(a) Convergence of the approximate posterior and true 
		posterior (with true $C^\star$) in averaged Wasserstein 
		distance for iterative PSMF. (b) Convergence of the mean $C_k$ 
		to $C^\star$. (c) Filter estimates given by the iterative PSMF 
		and the optimal filter.}
	\label{fig:Convergence}
\end{figure}

To gain insights in the convergence of our method, we have designed a simplified setup where the latent state trajectory is a one-dimensional random walk and observations are four-dimensional, and we have simulated a dataset consisting of size 1,000 where $C \in \bR^4$. We run the KF with the ground-truth value $C^\star$. We also run the iterative PSMF which also estimates $C$ as well as the hidden states. We have computed the distance between the sequence of optimal (Gaussian) filters constructed by the KF and the filters of the iterative PSMF in terms of the averaged Wasserstein distance over the path:
\begin{align}
{\overline{W}_2}(t) := \frac{1}{t} \sum_{k=1}^t W_2(p_\star(x_k | y_{1:k}),\tilde{p}(x_k|y_{1:k})).
\end{align}
We observe that the distance between the optimal and approximate filters over the entire path is uniformly bounded (see Fig.~\ref{fig:Convergence}(a)). We also observe that $C_k \to C^\star$ for this case, see Fig.~\ref{fig:Convergence}(b) and show the mean estimates are sufficiently close (Fig.~\ref{fig:Convergence}(c)).

More precisely, we simulate the following state-space model
\begin{align}
p(x_0) &= \NPDF(x_0;\mu_0,P_0), \\
p(x_k|x_{k-1}) &= \NPDF(x_k; x_{k-1}, Q), \\
p(y_k | x_k, C^\star) &= \NPDF(y_k; C^\star x_k, R),
\end{align}
where $C^\star \in \bR^4$ and $x_k \in \bR$, which leads to $y_k \in \bR^4$. In this case, the identifability problem is alleviated since $C^\star$ is a vector and we can test empirically whether the posterior provided by the PSMF for the states $p(x_k | y_{1:k})$ converges to the true posterior of the states $p_\star(x_k|y_{1:k})$.

Note that, the PSMF provides the filtering distribution of states as a Gaussian
\begin{align}
\tilde{p}(x_k| y_{1:k}) = \NPDF(x_k; \mu_k, P_k)
\end{align}
where $\mu_k,P_k$ are defined within Algorithm~\ref{algDF}. Since the data is generated from the model using $C^\star$, we also compute the optimal Kalman filter with $C^\star$ which we denote as $p_\star(x_k|y_{1:k})$. In order to test the convergence between the approximate filter provided by the PSMF $\tilde{p}(x_k|y_{1:k})$ and the true filter $p_\star(x_k|y_{1:k})$, we use the Wasserstein-2 distance which is defined as
\begin{align}
W_2(\mu,\nu) = \inf_{\Gamma \in \mathcal{C}(\mu,\nu)} \iint \|x-y\|^2 \Gamma(\md x, \md y)
\end{align}
where $\mathcal{C}(\mu,\nu)$ is the set of couplings whose marginals are $\mu$ and $\nu$ respectively. This Wasserstein-2 distance can be computed in closed form for two Gaussians, e.g., for $\mu = \NPDF(\mu_1,\Sigma_1)$ and $\nu = \NPDF(\mu_2,\Sigma_2)$, we have
\begin{align}
W_2(\mu,\nu)^2 = \|\mu_1 - \mu_2\|^2 + \Tr(\Sigma_1 + \Sigma_2 - 2(\Sigma_2^{1/2} \Sigma_1 \Sigma_2^{1/2})^{1/2}).
\end{align}
Hence, for a given sequence of filters $(\tilde{p}(x_k | y_{1:k}))_{k\geq 1}$ and $(p_\star(x_k | y_{1:k}))_{k\geq 1}$, we define the averaged Wasserstein distance for time $t$ as
\begin{align}
\overline{W}_2(t) = \frac{1}{t} \sum_{k=1}^t W_2(\tilde{p}(x_k | y_{1:k}),{p}_\star(x_k | y_{1:k})).
\end{align}
One can see from Fig.~\ref{fig:Convergence} that $\lim_{t\to\infty} \overline{W}_2(t) < \infty$ which implies that a convergence result can be proven for our method. We leave this exciting direction to future work.

\section{ADDITIONAL RESULTS FOR RECURSIVE PSMF}\label{app:recursivePSMF}
\begin{wrapfigure}{r}{0.4\textwidth}
    \vskip-\baselineskip 
  \centering%
  \includegraphics[width=0.4\textwidth]{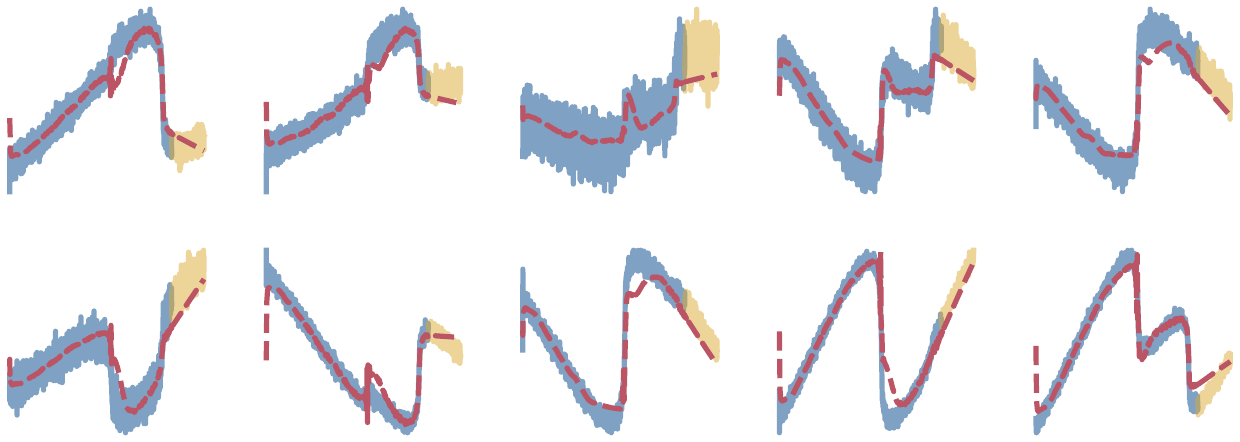}
  \caption{Recursive PSMF. Observed time series (\textcolor{SynBlue}{blue}) with unobserved future data (\textcolor{SynYellow}{yellow}) and the reconstruction from the model (\textcolor{SynRed}{red}).\label{fig:recursive}}
\end{wrapfigure}
In this section, we present an additional result using recursive PSMF to demonstrate the scalability of our method in a purely streaming setting.

Using the same setting of Sec.~4.1, we use a longer sequence ($n=4000$) with an additional prediction sequence of length $800$. This presents a challenging setting as we do not iterate over data and the algorithm observes the training data only once (i.e. the streaming setting). As can be seen from Fig.~\ref{fig:recursive}, recursive PSMF learns the underlying dynamics and has a successful out-of-sample prediction performance, even with a relatively long sequence into the future. This demonstrates the recursive version of our method can be used in a setting where iterating over data multiple times is impractical.
\end{document}